%% file: main.tex
\documentclass[twoside]{article}

\usepackage[accepted]{0_paper/include/aistats2023}
\input{0_paper/include/packages}


\begin{document}

\twocolumn[

\aistatstitle{A Unified Gaussian Process for Branching and Nested \\Hyperparameter Optimization}

\aistatsauthor{ Jiazhao Zhang \And Ying Hung \And  Chung-Ching Lin \And Zicheng Liu}
\aistatsaddress{ Rutgers University \\jiazhao.zhang@rutgers.edu \And  Rutgers University\\yhung@stat.rutgers.edu \And Microsoft \\chungching.lin@microsoft.com \And Microsoft\\zliu@microsoft.com } 
]

\input{0_paper/content/s0_abstra}
\input{0_paper/content/s1_introd}

\input{0_paper/content/s3_method}

\input{0_paper/content/s4_experi}

\input{0_paper/content/s5_conclu}

\clearpage
\appendix
\input{1_supplementary/content/s8_supple_bkup.tex}

\bibliography{0_paper/main-bib}

\end{document}

%% file: 0_paper/include/packages.tex
\usepackage{url}            
\usepackage{booktabs}       
\usepackage{nicefrac}       

\usepackage[dvipsnames]{xcolor}
\usepackage[colorlinks=true,linkcolor=red,citecolor=green]{hyperref}

\usepackage{lipsum} 
\usepackage{graphicx}
\usepackage{makecell}
\usepackage{multirow,multicol}

\usepackage{amsmath,amsthm,amsfonts,dsfont}
\usepackage{bm,bbm}
\usepackage{enumitem}
\usepackage{physics}
\usepackage{float}

\usepackage{dblfloatfix}
\usepackage{caption,subcaption}

\usepackage{natbib}
\setcitestyle{authoryear,open={(},close={)}}

\def\mathbold{\boldsymbol}
\def\bR{\mathbold{R}}
\def\bY{\mathbold{Y}}

\def\bx{\mathbold{x}}

\newtheorem{definition}{Definition}
\newtheorem{theorem}{Theorem}

%% file: 0_paper/content/s0_abstra.tex
\begin{abstract}

    Choosing appropriate hyperparameters plays a crucial role in the success of neural networks as hyperparameters directly control the behavior and performance of the training algorithms. 
    To obtain efficient tuning, Bayesian optimization methods based on Gaussian process (GP) models are widely used. Despite numerous applications of Bayesian optimization in deep learning, the existing methodologies are developed based on a convenient but restrictive assumption that the tuning parameters are independent of each other. 
    However, tuning parameters with conditional dependence are common in practice. In this paper, we focus on two types of them: branching and nested parameters. Nested parameters refer to those tuning parameters that exist only within a particular setting of another tuning parameter, and a parameter within which other parameters are nested is called a branching parameter. 
    To capture the conditional dependence between branching and nested parameters, a unified Bayesian optimization framework is proposed. 
    The sufficient conditions are rigorously derived to guarantee the validity of the kernel function, and the asymptotic convergence of the proposed optimization framework is proven under the continuum-armed-bandit setting. 
    Based on the new GP model, which accounts for the dependent structure among input variables through a new kernel function, higher prediction accuracy and better optimization efficiency are observed in a series of synthetic simulations and real data applications of neural networks.
  Sensitivity analysis is also performed to provide insights into how changes in hyperparameter values affect prediction accuracy.

\end{abstract}

%% file: 0_paper/content/s1_introd.tex
\section{INTRODUCTION} 

Tuning deep learning hyperparameters is a tedious yet critical task, as the performance of an algorithm can be highly dependent on the choice of hyperparameters, including architectural choices and regularization hyperparameters. As modern applications of deep learning algorithms become increasingly complex, it is crucial to have a global optimization approach that not only can find the optimal tuning efficiently but also quantifies the impacts of various tuning parameters so that the computational resource can be located more effectively.
A commonly used approach for hyperparameter optimization is the Bayesian optimization method. Based on a Gaussian process (GP) prior, Bayesian optimization strategies sequentially include new observations by an expected improvement criterion which takes into account the trade-off between exploration and exploitation.\citep{jones1998efficient, shahriari2015taking, frazier2018tutorial}.  

Practical applications of Bayesian optimization in deep learning must be able to easily handle tuning problems ranging from just a few to many dozens of hyperparameters \citep{bergstra2012random,sutskever2013importance,mendoza2016towards,falkner2018bohb,smith2018disciplined,cohn1996active,cohn1996}. 
Different tuning problems in deep learning give rise to different types of parameters, including many of them with conditional dependence, each of which needs to be handled effectively by a practical Bayesian global optimization method. 
The existing works in Bayesian optimization are limited to a convenient but unrealistic assumption in which all the tuning parameters are mutually independent. However, tuning parameters with conditional dependence commonly occur in practice. In this paper, we focus on two types of them: branching and nested parameters. Nested parameters refer to those tuning parameters that exist only within a particular setting of another tuning parameter, and a parameter within which other parameters are nested is called a branching parameter. Branching and nested hyperparameters are often involved in deep learning algorithms. As an example, two sets of them are illustrated in Table \ref{tab:ExBNinCnn}. In Table \ref{tab:ExBNinCnn}, the tuning parameter, \textit{network type}, is a branching parameter in the convolution neural network algorithm with two common choices: \textit{ResNet} \citep{he2016deep} and \textit{MobileNet} \citep{howard2017mobilenets}. The parameter \textit{Depth} is a nested tuning parameter that exists only if \textit{ResNet} is chosen, and the parameter \textit{Width Multipliers} with three choices is nested within the setting of \textit{MobileNet}. 

\begin{table}[!htb]
    \small
    \centering
    \caption{Examples Of Branching And Nested Tuning Parameters In A Convolution Neural Network.}
    \begin{tabular}{l|l|l}
    \toprule
        \multicolumn{2}{c|}{Branching} & \multirow{2}{*}{Nested variables}\\
         Variables & Categories & \\
    \midrule[\heavyrulewidth]
        \multirow{2}{*}{{Network type}} & {ResNet}& {Depth} = \{18,34,50,101\} \\
          & {MobileNet}& {Multipliers} $\in$ (0,1) \\
    \midrule[\heavyrulewidth]
        \multirow{2}{*}{{Optimizer}} & {SGD}& {Scheduler} = \{{Cyclic} or  {Cosine}\} \\
         & {Adam}& {Scheduler} = \{{Step} or {Cosine}\} \\
    \bottomrule
    \end{tabular}

    \label{tab:ExBNinCnn}
\end{table}

The importance of branching and nested structure is first pointed out by \citet{taguchi1987system} and \citet{phadke1995quality} in conventional statistical inference. It is crucial to incorporate the potential interaction between branching and nested parameters in a statistical model because the nested parameters differ with respect to the settings of the branching parameters, and thus their effects can change. However, most of the existing Bayesian optimization approaches \citep{snoek2012practical, HutHooLey11-censoring, bergstra2011algorithms, hutter2011sequential} ignore this complex structure in the construction of expected improvement criteria, which can lead to an inefficient search and a  misleading inference based on the fitted model.

To address this problem, a unified Bayesian optimization procedure called
B$\&$N (branching and nested hyperparameter optimization)
is proposed. It is based on a new GP model, which accounts for the dependent structure among input variables through a new kernel function. Unlike direct applications of the conventional kernel functions where the same correlation parameter is assumed for nested tuning parameters, the new kernel function allows the correlation parameters for nested variables to change depending on the setting of the corresponding branching parameter. As a result, different impacts of the nested parameter settings can be properly captured and therefore leads to an efficient search and a reliable inference. Based on the new kernel function, a GP model can be constructed, and the corresponding sensitivity analysis can be performed. New untried tuning parameter settings are sequentially included in an expected improvement criterion, which effectively balances exploration and exploitation. Accordingly, the global optimal setting for tuning parameters could be achieved.

To summarize, this paper has four main contributions: 
\begin{itemize}
\item[(1)] We introduce a new kernel function for Gaussian process models to incorporate a commonly occurring conditional dependence between branching and nested tuning parameters and show the sufficient conditions for a valid kernel induced by the new kernel function. 
\item[(2)] Based on the new kernel, we propose a unified Bayesian optimization framework B$\&$N (Branching and Nested hyperparameter optimization), which is broadly applicable to different types of tuning parameters. By incorporating the conditional dependence, the resulting procedure is more efficient in obtaining the global optimal as compared with the existing methods. 
\item[(3)] We prove that the proposed Bayesian optimization procedure converges in its reproducing-kernel Hilbert space (RKHS), and the asymptotic convergence rate is derived under the continuum-armed-bandit setting \cite{agrawal1995continuum, kleinberg2004nearly, bubeck2008online, kleinberg2008multi}.
\item[(4)] Based on the new GP model, sensitivity analysis can be performed to provide a better understanding of the impacts of various hyperparameter tuning on the accuracy of deep learning.  
\end{itemize}

%% file: 0_paper/content/s3_method.tex
\section{A UNIFIED GAUSSIAN PROCESS FOR BRANCHING AND NESTED TUNING PARAMETERS}

\subsection{Previous Work on Bayesian Optimization}
Bayesian optimization is a global optimization strategy developed based on stochastic Gaussian process priors to optimize an unknown function, denoted by $f(\bx)$, that is expensive to evaluate \citep{zhigljavsky2007stochastic,shahriari2015taking}.  Such problems are often described as  ``black  box''  optimization  problems, where obtaining  data  from  the  ``black  box'' in the current setting requires a computationally intensive deep neural network algorithm and the goal is to find the optimal setting $\bx^* = \text{argmax}_{\bx\in\mathcal{X}} f(\bx)$, where $\bx\in\mathcal{X}\subset\mathbb{R}^d$ refers to the $d$-dimensional tuning parameters involved. 
The function $f(\bx)$ can be generally defined based on the interest of specific applications. In this paper, we illustrate the proposed idea by assuming the prediction accuracy as the objective. It can be easily generalized to other objective functions, such as architecture search on convolutional neural network topology \citep{zoph2016neural, real2017large, tan2019mnasnet, ying2019bench} or joint architecture-recipe (i.e., training hyperparameters) search \citep{dai2020fbnetv3}. To optimize the unknown function $f$, a Bayesian approach with Gaussian process prior is often applied. Assume the unknown function is a realization from a stochastic process
$$
Y(x)=\mu(\bx)+Z(x),
$$
where $\mu(\bx)=m(\bx)^T\*{\beta}$ with prespecified variables $m(\bx)$, $Z(x)$ is a a weak stationary Gaussian process with mean 0 and covariance function $\sigma^2R_{\*\theta}$, and  $\*\theta$ is the unknown correlation parameter.
Let $(\bx_1, \bx_2,\dots, \bx_n)^T$ be the input matrix with sample size $n$ and $\bY_n=(y_1, y_2,\dots, y_n)^T$ be the corresponding outputs.

Based on the maximum likelihood approach, the parameters, $\*\beta$, $\sigma^2$ and $\*\theta$ can be estimated by: 
\begin{align}
      \widehat{\*\beta} &= (\*M_n^T{\bR}^{-1}\*M_n)^{-1}\*M_n^T{\bR}^{-1}\bY_n, \\
      \widehat{\sigma^2} &= \frac{1}{n}(\bY_n-\*M_n\widehat{\*\beta})^T{\bR}^{-1}(\bY_n-\*M_n\widehat{\*\beta}),\\
      \widehat{\*\theta} &= \text{argmax}_{\*\theta} \left[-\frac{n}{2}\log \widehat{\sigma^2} -\frac{1}{2}\log\abs{\bR}\right],\label{eq:corr est}
\end{align}    
where $\*M_n$ is the design matrix, $\*M_n\*{\beta}=( m(\bx_1)^T\*\beta, \dots, m(\bx_n)^T\*\beta)^T$, and $\bR$ is a $n$-by-$n$ correlation matrix whose $ij$th element is $R_{\theta}(\bx_i, \bx_j)$.  

For an untried setting $\*x_0$, the predictive distribution is Gaussian with the best linear unbiased predictor $$\widehat{y(\bx_0)}= m(\bx_0)^T\widehat{\*\beta}+\widehat{r}_0^T\widehat{\bR}^{-1}\left(\bY_n-\*M_n\widehat{\*\beta}\right)$$ as the mean, where $\widehat{r}_0$ is a vector of length $n$ with the $i$th element  $R_{\theta}(\bx_0, \bx_i)$, and variance
\begin{align*}
    \widehat{s^2_0} = \widehat{\sigma^2}\bigg[ 1 - \widehat{r_0}^T\widehat{\bR}^{-1}\widehat{r}_0 + \Delta m_0^T\left(\*M^T\widehat{\bR}^{-1}\*M\right)^{-1}\Delta m_0\bigg]
\end{align*}
where $\Delta m_0= m(x_0)-\*M^T\widehat{\bR}^{-1}\widehat{r_0}$.
 
Based on the fitted GP as the surrogate to the underlying truth $f$, Bayesian optimization proceeds by selecting the next point to be evaluated using an \textit{acquisition function}. Various acquisition functions are discussed in the literature \citep{srinivas2009gaussian, snoek2012practical,jones1998efficient,bull2011convergence, wang2014theoretical,jones1998efficient}. We focus on a commonly used function called \textit{Expected Improvement}.
The expected improvement $EI(\bx)$ can be written as 
\begin{align*}
\mathbb{E}\left\{(y(\bx) - y_{max})_+\right\}&=\widehat{s_0}\phi\left( \frac{\widehat{y}(\bx)-y_{n,max}}{\widehat{s_0}} \right)\\
&+(\widehat{y}(\bx)-y_{max})\Phi\left( \frac{\widehat{y}(\bx)-y_{max}}{\widehat{s_0}} \right),
\end{align*}
where $(\cdot)_+=max(\cdot,0)$, $y_{max}= \max_{1\leq i\leq n} y_i$ is the current best observation, and $\phi$ and $\Phi$ are the pdf and cdf of standard normal distribution. 
To tackle the nonconvergence issue due to GP parameter estimation in the conventional EI, \citet{bull2011convergence} proposed a modification called $\epsilon-$greedy $\widetilde{\text{EI}}$. They use a restricted MLE which solves (\ref{eq:corr est}) and satisfies $\theta_i^L\leq\widehat{\theta_{i}}\leq\theta_i^U$ for $1\leq i\leq d$ and replace 
the MLE of $\widehat{\sigma^2}$ by $ (\bY_n-\*M_n\widehat{\*\beta})^T\widehat{\bR}^{-1}(\bY_n-\*M_n\widehat{\*\beta})$ and in addition to the selections by maximizing EI in each iteration, a randomly chosed untried setting is selected with probability $\epsilon$.

In deep learning literature, there are numerous applications of Bayesian optimization to achieve automatic tuning for quantitative parameters \citep{snoek2012practical, hernandez2015predictive, frazier2018tutorial, ru2018fast, alvi2019asynchronous}. Recent studies have shown an increasing interest in the extension of the Bayesian optimization to qualitative or categorical parameters, including \citep{hutter2011sequential}, \citep{golovin2017google}, \citep{gonzalez2016batch}, \citep{garrido2020dealing}, \citep{nguyen2019bayesian}, and \citep{ru2020bayesian}.
However, to the best of our knowledge, most of the existing Bayesian optimization methods are based on Gaussian process priors which assume the independence between the tuning parameters. This assumption is convenient for the construction of Gaussian process correlation functions but is commonly violated in practical applications of deep learning algorithms. With the increasing complication of deep learning applications, the tuning parameters are often conditionally dependent, such as the branching and nested parameters shown in Table \ref{tab:ExBNinCnn}. Naive extensions of the existing methods by separate analysis of each branching and nested parameter combinations are inefficient for the search of the global optimal and the inference resulting from the fitted model can be misleading \citet{hung2009design}. Therefore, a unified framework that can efficiently take into account the conditional dependence is called for.

\subsection{A New Gaussian Process Model for Branch and Nested Parameters}

To take into account dependence between tuning parameters, we introduce a new kernel function for Gaussian process models which is essential to the proposed B$\&$N framework.
Consider three types of variables involved in the optimization problem: (i) $\mathbf{w}=(w_1,\ldots,w_d)$ which are $d$ quantitative variables, (ii) $\mathbf{z}=(z_1, z_2, \ldots, z_q)$ which are $q$ branching variables often being qualitative and each of which has $l_k$ categorical levels, and (iii) $\mathbf{v}^{z_k}=(v_1^{z_k}, v_2^{z_k}, \ldots, v_{m_k}^{z_k})$ for $k=1,\ldots,q$, which are $m_k$ variables nested within $z_k$, and $\mathbf{v}=\{\{\mathbf{v}_j^{z_k}\}^{m_k}_{j=1}\}^q_{k=1}$ represents all the nested variables. Thus, there are $p$ variables in the  optimization  problem, where $p=d+q+\sum_{k=1}^q m_k$.

We first develop a Gaussian process model with a new kernel function that involves these three types of variables. For any two inputs $\mathbf{x}=(\mathbf{w},\mathbf{z},\mathbf{v})$ and $\mathbf{x}'=(\mathbf{w}',\mathbf{z}',\mathbf{v}')$, we consider the product kernel
\begin{equation}\label{eq:productkernel}
    R(\mathbf{x},\mathbf{x}')=R_{\boldsymbol{\theta}}(\mathbf{w},\mathbf{w}')R_{\boldsymbol{\gamma}}(\mathbf{z},\mathbf{z}')R_{\boldsymbol{\phi}}(\mathbf{v},\mathbf{v}'),
\end{equation}
where $\boldsymbol{\theta},\boldsymbol{\gamma}$ and $\boldsymbol{\phi}$ are the hyperparameters of the kernels. For $R_{\boldsymbol{\theta}}(\mathbf{w},\mathbf{w}')$, we consider typical kernels for quantitative variables, such as exponential and Mat\'ern kernels \citep{stein2012interpolation}. For qualitative variables $\mathbf{z}$, one popular choice of the kernel is
\begin{equation}
    \label{eq:branchkernel}
    R_{\boldsymbol{\gamma}}(\mathbf{z}, \mathbf{z}')=\prod_{k=1}^q\exp\left\{ -\gamma_k\mathds{1} \{ z_{k}\neq z'_{k} \}  \right\},
\end{equation}
where $\mathds{1}\{\cdot\}$ is an indicator function \citep{qian2008gaussian,han2009prediction,zhou2011simple,zhang2015computer,huang2016computer,deng2017additive}. For nested variables, the conventional kernel functions using only one  hyperparameter for each nested variable is not desirable because a nested variable can represent completely different effects according to the setting of the branching variable.   
Therefore, we consider the following kernel function for nested variables.

\begin{definition}
\label{def:nestkernel}
For any two nested variables $\mathbf{v}$ and $\mathbf{v}'$, denote $R_{\boldsymbol{\phi}}(\mathbf{v}, \mathbf{v}')=\prod_{k=1}^q R_{\boldsymbol{\phi}_k}(\mathbf{v}^{z_k}, \mathbf{v}'^{z'_k})$ and
       \begin{align*}
            R_{\boldsymbol{\phi}_k} = \exp\left\{-\sum_{b=1}^{l_k} \left(\mathds{1}{\left\{ z_k=z'_k=b \right\}} \sum_{j=1}^{m_k}\phi^b_{kj}d(v_j^b,v'^b_j)\right) \right\},
        \end{align*}
        where $d(v_j^b,v'^b_j)=|v_j^b-v'^b_j|$ if $\mathbf{v}^{z_k}$ is quantitative and $d(v_j^b,v'^b_j)=\mathds{1}\{ v_j^b\neq v'^b_j \}$ if $\mathbf{v}^{z_k}$ is qualitative.
\end{definition}

The concept of branching and nested variables  is first introduced by \citet{hung2009design}. However, there is no theoretical assessment for the validity of the corresponding correlation function. Furthermore, based on the new correlation function in a Gaussian process prior, the convergence property of the resulting active learning is also unclear.

To study the theoretical convergence properties of the proposed B$\&$N procedure, we first need to show that the kernel function in Definition \ref{def:nestkernel} leads to a valid reproducing-kernel Hilbert space (RKHS). By Moore-Aronszajn Theorem \citep{aronszajn1950theory}, there exists a unique RKHS corresponding to a properly defined kernel function. Therefore, a crucial first step is to find sufficient conditions for a valid kernel induced by the new kernel function. These theoretical conditions are obtained by the following theorem, illustrated by qualitative branching and nested variables.
    \begin{theorem}
    \label{thm:validkernel}
    Suppose that there are $g^b_j$ levels in the nested variable $v_j^{b}$ which is nested within the branching variable $z_k=b$, for any $b\in\{1,2,\dots,l_k\}$. The kernel function in (\ref{eq:productkernel}) is symmetric and positive definite if the hyperparameter $\boldsymbol{\phi}_k$ satisfy:
    \begin{align}
               \min_{b} \left[ \exp(-\phi^b_{kj})+\left(1-\exp(-\phi^b_{kj})\right)/g^b_j\right]
               \geq 
               \exp(-\gamma_k),
     \end{align}          
    for all $j \in \{1,2,\ldots, m_k\}$ and $k\in\{1,2,\ldots, q\}$. 
    \end{theorem}

{\bf Remark 1}
Theorem \ref{thm:validkernel} implies a stronger but more intuitive sufficient condition for the validity of the kernel function, i.e.,
          $$       \exp(-\phi_{kj}^b)\geq\exp(-\gamma_k).
            $$
The intuition behind this sufficient condition in Remark 1 for the correlation parameters is that observations from different branching groups should be less correlated than those from the same group.
This remark is practically useful in the search of MLEs in the proposed B$\&$N approach. The MLEs can be efficiently found by a constrained optimization approach and the validity of the estimated parameters is directly guaranteed.

Based on the new kernel function, the branching and nested variables can be incorporated into a Gaussian process prior, which then can be used to construct an expected improvement criterion for active learning. 
Given the current best $y_{\max}$, the proposed B$\&$N method adaptive includes new observations at $\mathbf{x}^*$ which maximize the following EI criterion: %
\begin{align*}
    \mathbb{E}\left[(f(\mathbf{x})-y_{\max})_+\right]
    =&\left(y^{\mu}(\mathbf{x})-y_{\rm{max}}\right)\Phi\left(\frac{y^{\mu}(\mathbf{x})-y_{\rm{max}}}{\sqrt{y^{\sigma}(\mathbf{x})}}\right)\\
    &+ \sqrt{y^{\sigma}(\mathbf{x})}\phi\left(\frac{y_{\rm{max}}-y^{\mu}(\mathbf{x})}{\sqrt{y^{\sigma}(\mathbf{x})}}\right),
\end{align*}
where $y^{\mu}(\mathbf{x})$ and $y^{\sigma}(\mathbf{x})$ are the posterior mean and variance obtained by
\begin{align*}
    y^{\mu}(\mathbf{x})=\mu(\mathbf{x})+\mathbf{k}_n(\mathbf{x})^T\mathbf{K}_n^{-1}\left(\mathbf{y}_n-\mu(\mathbf{X}_n)\right)
\end{align*}
and
\begin{align*}
    y^{\sigma}(\mathbf{x})=\sigma^2-\sigma^2\mathbf{k}_n(\mathbf{x})^T\mathbf{K}_n^{-1}\mathbf{k}_n(\mathbf{x}) 
\end{align*}
with $\mathbf{k}_n(\mathbf{x})=(K(\mathbf{x},\mathbf{x}_1),\ldots,K(\mathbf{x},\mathbf{x}_n))$ and $\mathbf{K}_n=(K(\mathbf{x}_i,\mathbf{x}_j))_{1\leq i,j\leq n}$, and $\mathbf{y}_n=(y_1,\ldots,y_n)^T$.

\subsection{Convergence Results}

The following theorem provides a bounded simple regret for the proposed approach, which is analogous to the results for quantitative variables in \citet{bull2011convergence}.

\begin{theorem}{Assume that $f$ follows a Gaussian process with the kernel function defined in (\ref{eq:productkernel}), where $R_{\boldsymbol{\theta}}$ is a Mat\'ern kernel with the smoothness parameters $\nu>0$ and some $\alpha\geq 0$ depends on $\nu$. Let $y^*_n=max_{1\leq i\leq n}y_i$, we have
\begin{align*}
    \sup_{\|f\|_{\mathcal{H}(\mathcal{S})}\leq S}\mathbb{E} & \left\{y^*_{n}-\max_{\mathbf{x}\in\mathcal{X}} f(\mathbf{x})|D_n\right\}\\
    &= \mathcal{O}\left(L^{\nu/d}(n/\log{n})^{-\nu/{d}}(\log n)^{\alpha}\right).
\end{align*}
}
\end{theorem}

%% file: 0_paper/content/s4_experi.tex
\section{SIMULATION STUDIES}
\label{sec:sim}

A numerical simulation is conducted in this section to demonstrate the performance of the proposed framework and compare the efficiency with the existing alternatives in the literature. 
The simulations are conducted from the following synthetic function with two continuous variables and a pair of branching and nested variables: 
\begin{align}
\label{fun:fun2d}
    f(x_1,&x_2, z, v) = \frac{v}{2}\exp\{-(x_1-c_1 )^2\} \nonumber \\
        &+ \frac{2}{v}\exp\{-\frac{1}{10}(x_1-c_2)^2\}
        + \frac{1}{x_2^2+1}
        + z,
\end{align}
where $x_1\in[-10, 10]$, $x_2\in[-5, 5]$. The branching variable $z$ has two settings, $z=1$ or 2. There are three choices for the nested variables when $z=1$ and two choices when $z=2$. Based on different branching and nested parameter settings, the values of $c_1$ and $c_2$ are specified according to the functions in Table \ref{tab:BNinSynf}.

\begin{table}[!htb]
    \centering
    \caption{The Settings For Branching And Nested Parameters In (\ref{fun:fun2d}).}
    \begin{tabular}{cccc}
    \toprule  
    \textbf{branching} & \textbf{Nested} & $c_1$ & $c_2$\\
    \midrule
    $z=1$ & $v=\{1,2,3\}$  & $3-0.5\times v$ & $5-v$\\
    $z=2$ & $v=\{1,2\}$  & $-1+v$ & $7-v$\\ 
    \bottomrule
    \end{tabular}    

    \label{tab:BNinSynf}
\end{table}

\vspace{10pt}
\captionsetup[figure]{font=small}
\begin{figure}
    \centering
    \includegraphics[scale=0.13]{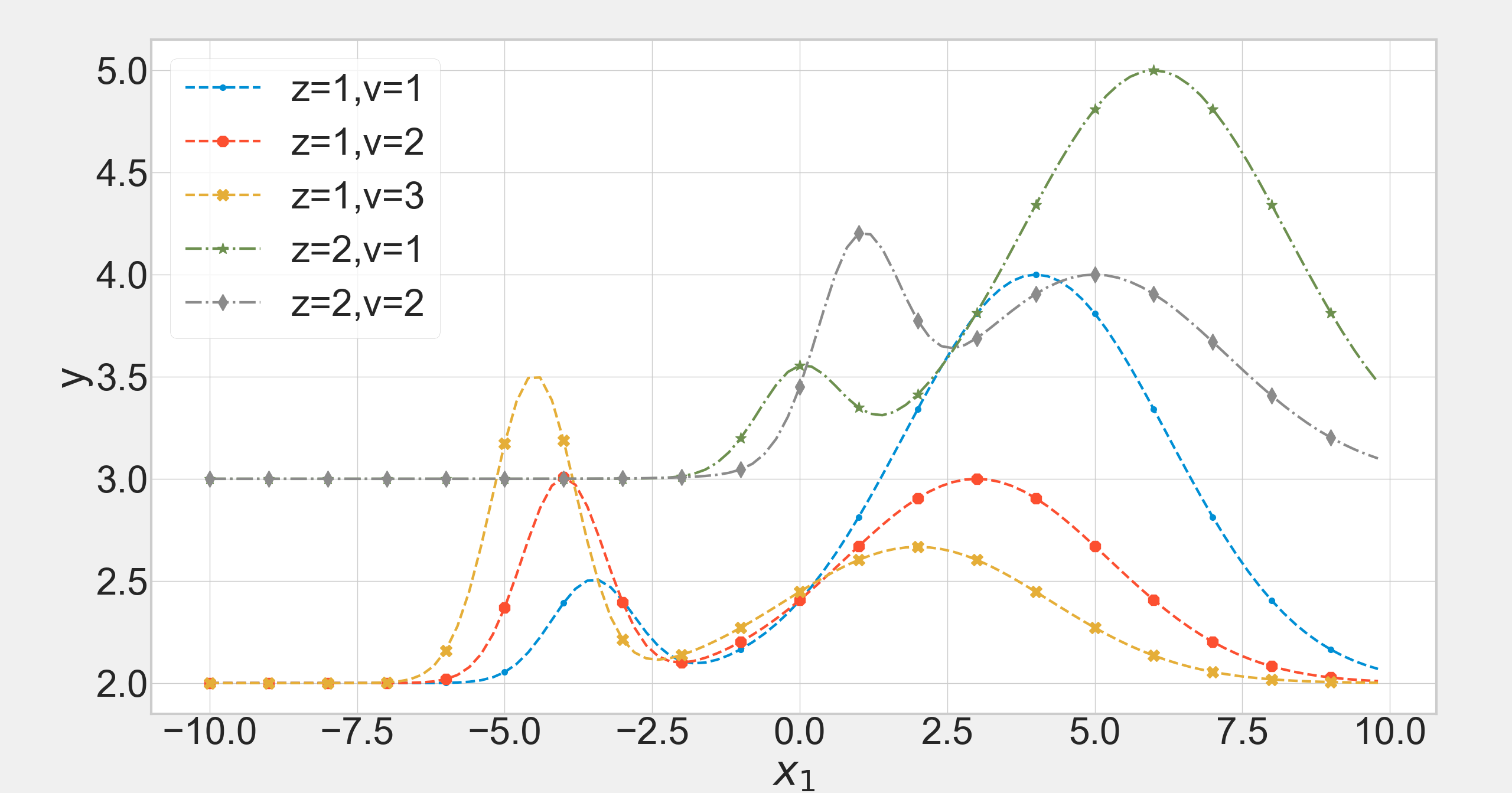}
    \caption{An illustration of the synthetic function with one branching parameter $z$, one correspond nested parameter $v$, and two quantitative parameters, $x_1$ and $x_2$.For the five different combinations of the branching and nested parameters, this picture shows the projected function onto $x_1$ at $x_2=0$.}
    \label{fig:1d vis of fun2d}%
\end{figure}

\captionsetup[figure]{font=small}
\begin{figure*}[!htb]
\begin{center}
\subcaptionbox{Sequential procedure.}{
\includegraphics[scale=0.10]{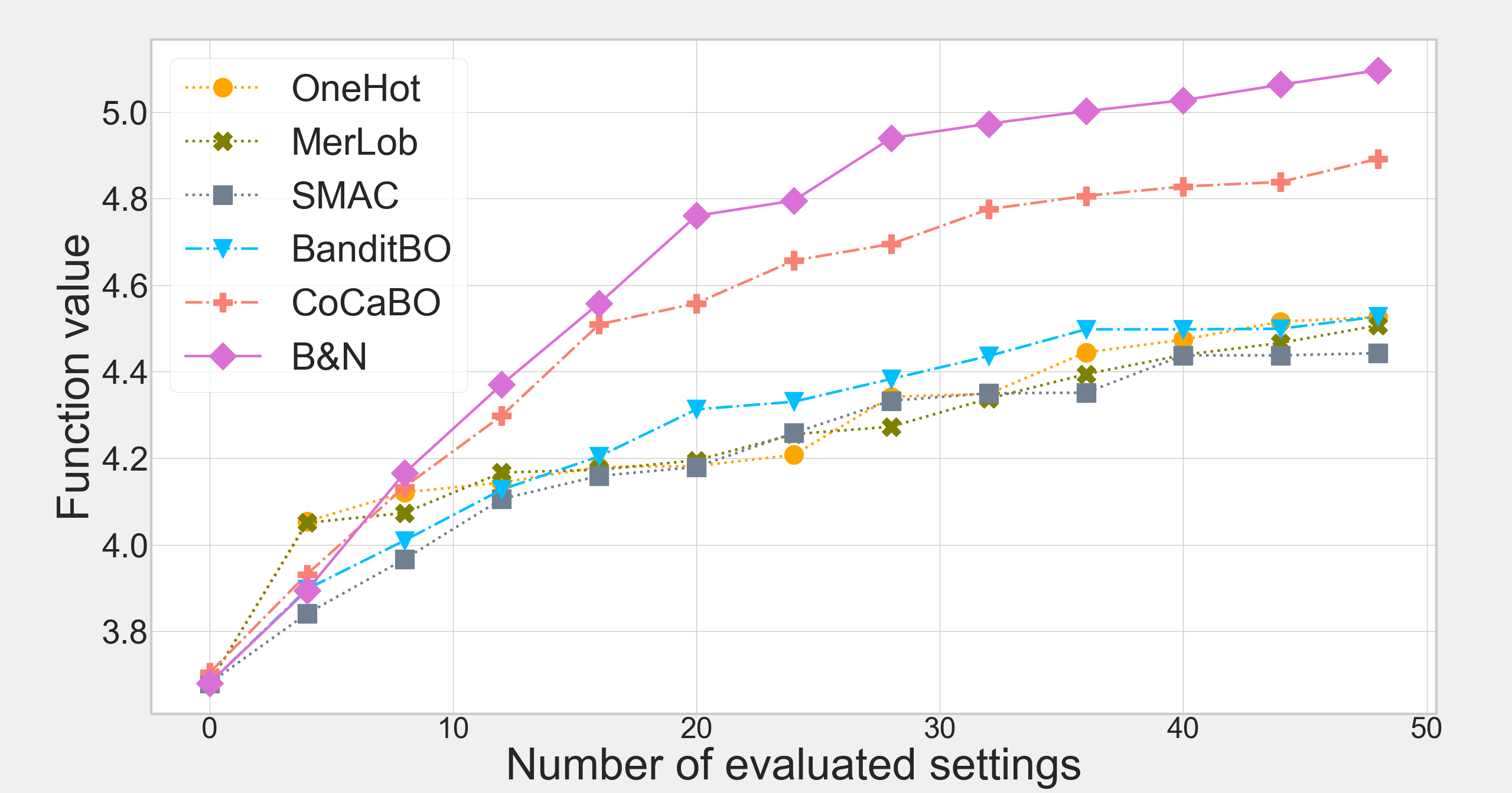}
\includegraphics[scale=0.10]{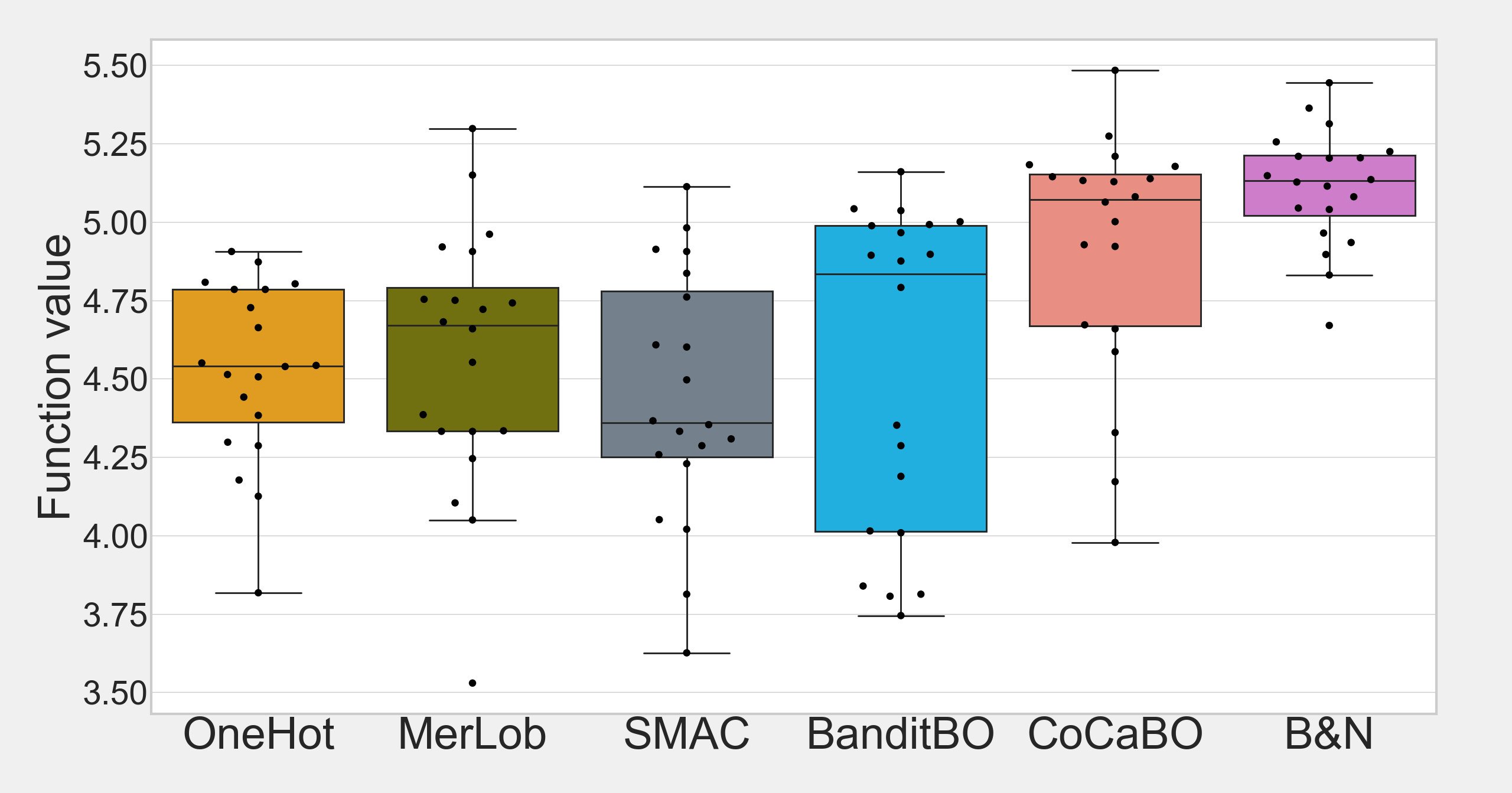}
}

\subcaptionbox{Batch procedure with batch size 5.}{
\includegraphics[scale=0.10]{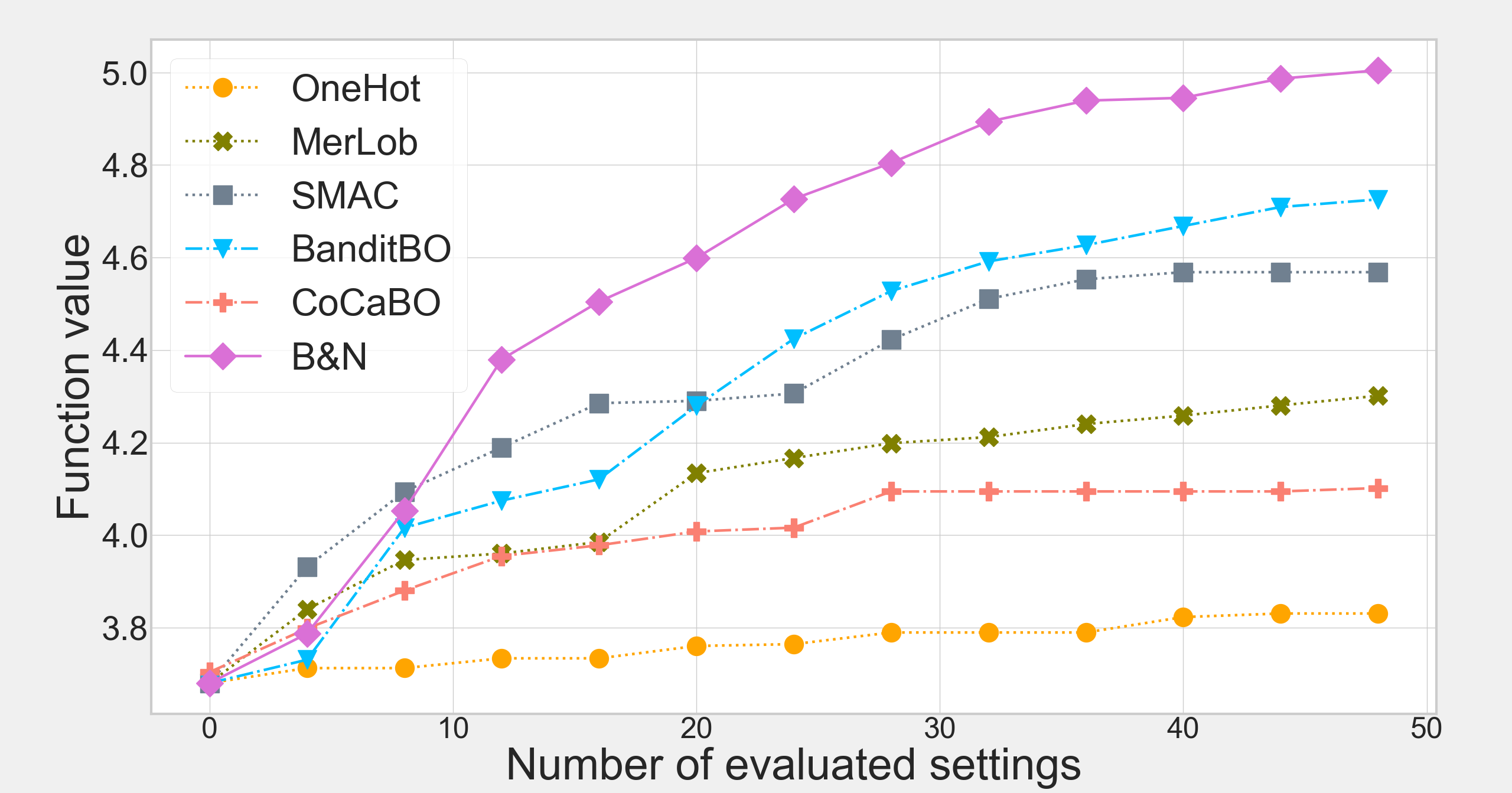}
\includegraphics[scale=0.10]{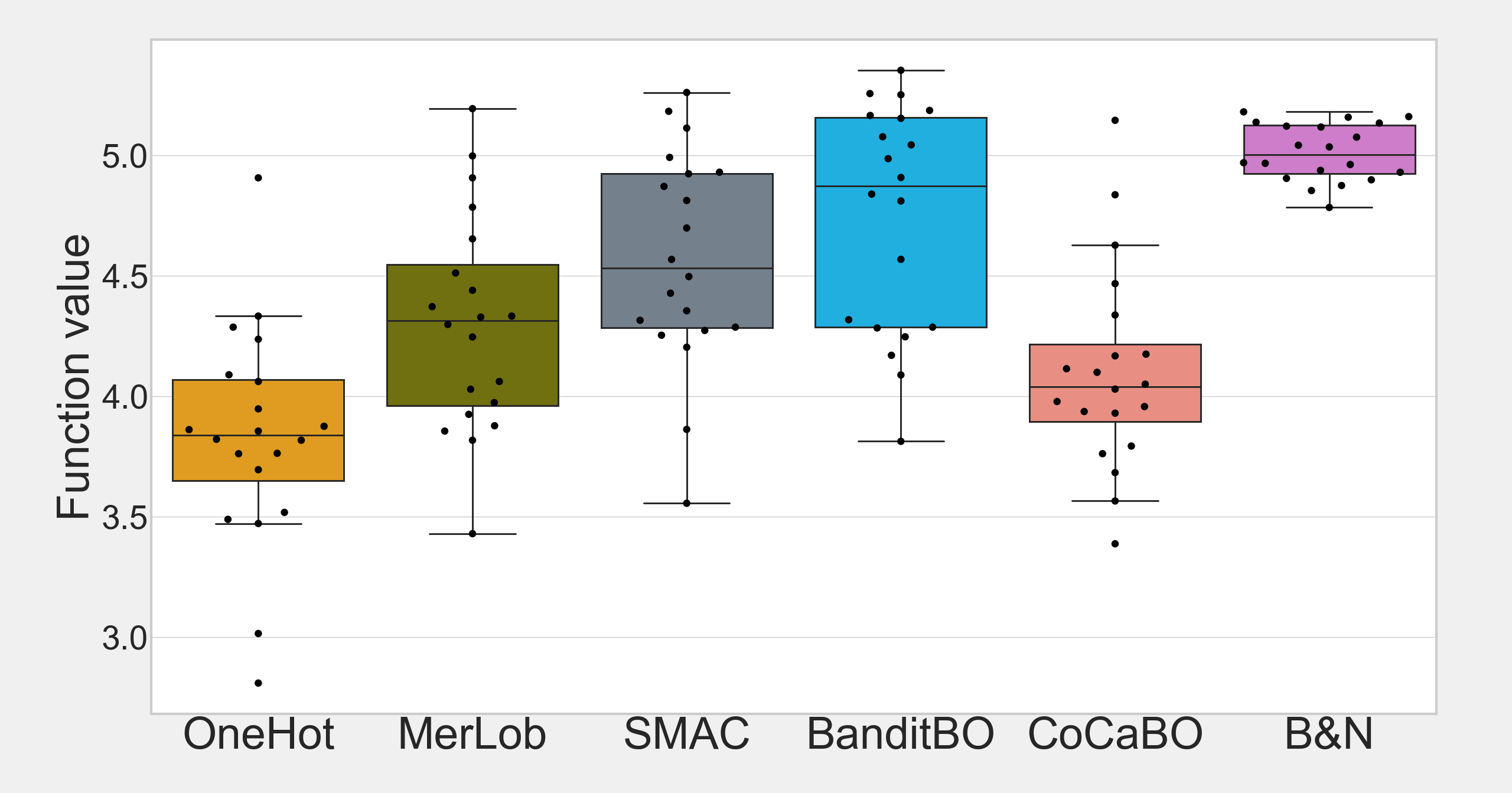}
}
\end{center}
    \centering

    \caption{Compare B$\&$N with five existing methods based on synthetic data. Starting from 10 randomly generated initials, two adaptive procedures, sequential and batch, are conducted to include additional 50 observations. The results based on the sequential procedure are shown in the upper panel and those for the batch procedure with batch size 5 are shown in the lower panel. The plots on the left
    show the best function value as the search progresses. B$\&$N appears to be able to identify the optimal setting and reaches the global optimal much faster than the other methods. For each method, the box-plots on the right are the final optimal results summarized from 20 replicates at the end of the search. The average optimal results found by B$\&$N consistently outperform the existing methods with a smaller variation.}
\label{fig:synfExpResult} %

\end{figure*} 

To visualize the simulation function, we illustrate the responses as a function of $x_1$ with respect to the five different combinations of the branching and nested parameters in Figure \ref{fig:1d vis of fun2d}. Each function in Figure \ref{fig:1d vis of fun2d} is a two-Gaussian mixture distribution with mean values $c_1$ and $c_2$. According to (\ref{fun:fun2d}), the setting of $z$ serves as the baseline and, depending on its setting, the nested parameters play different roles through the settings of $c_1$ and $c_2$. A normally distributed random error with zero mean and standard deviation 0.2 is incorporated into the simulation to examine the robustness of the proposed method. The goal here is to efficiently find the optimal parameter setting at {$(x_1,x_2,z,v)=(6,0,2,1)$} with the corresponding global maximum functional value $f=5$.

To the best of our knowledge, there is no existing work designed directly to address the branching and nested structure in hyperparameter optimization. Therefore, we consider five recent methods which are mainly developed for categorical parameters with slight modifications as alternatives, including the One-hot Encoding \citep{golovin2017google, gonzalez2016batch} denoted by \textit{OneHot}, an extension of the One-hot Encoding proposed by
\citet{garrido2020dealing} denoted by \textit{MerLob}, the \textit{SMAC} proposed
\citet{hutter2011sequential}, the \textit{BanditBO} proposed by \citet{nguyen2019bayesian}, and the \textit{CoCaBO} proposed by \citet{ru2020bayesian}. For all baseline methods, we used public codes released by authors. In this simulation, the five distinct combinations of branching and nested parameters are viewed as 5 different categories so that the existing methods can be applied. For example, BanditBO fits an independent Gaussian process model for each category. In the implementation of CoCaBO, the correlations due to different categories are computed by 
$R_c(h, h^*)=\mathbb{I}\{h=h^*\},$ 
where $\mathbb{I}\{\cdot\}$ is the indicator function and $h$, $h^*\in\{1,2,\dots 5\}$. Furthermore, the correlation between two inputs $(x,h)$ and $(x^*,h^*)$ is given by 
$(1-\lambda)\left[r_s(\bx, \bx^*)+r_c(h, h^*)\right] + \lambda r_s(\bx, \bx^*)r_c(h, h^*)$
with $\lambda\in[0,1]$. Thus, the resulting correlation is either $(1-\lambda)r_s(\bx, \bx^*)$ when $h\neq h^*$ or $r_s(\bx, \bx^*)+(1-\lambda)$ when $h=h^*$. Thus, the correlation function proposed by CoCaBO is a special case of Definition \ref{def:nestkernel} when there are only categorical parameters.

To perform the comparison, the five methods are applied to the same initial design with 10 parameter settings randomly chosen from the parameter space. Two active learning procedures, sequential and batch, are conducted to include an additional 50 observations. The sequential procedure includes additional points one-at-a-time and the batch procedure includes five additional points in each iteration with a total of 10 iterations. 
Based on 20 replicates of the two procedures, the performance of the six methods is compared by their efficiency in finding the global optimal.

Using the sequential procedure, the search results are demonstrated on the upper left panel in Figure \ref{fig:synfExpResult} and the final optimal solutions obtained by each method are summarized on the upper right panel. From the figures, it appears that the proposed B$\&$N method identifies the optimal setting and reaches the global optimal much faster than the other five alternatives. Based on the box plot, it shows that the average optimal result found by B$\&$N consistently outperforms the other methods with a smaller variation.

A similar analysis is conducted for the batch procedure and the results are summarized in the lower panels in Figure \ref{fig:synfExpResult}. Overall, the results of B$\&$N still outperform the rest of the methods in terms of search efficiency and the average optimal solution obtained in the final step. It also appears that the sequential procedure works slightly better than the batch procedure in this example, where the average optimal value is 5.11 for the sequential procedure and 5.01 for the batch procedure.

\section{APPLICATIONS IN DEEP NEURAL NETWORK}\label{sec:cnn}

\subsection{Hyperparameter Optimization}

\begin{table}[!htb]
    \small
    \caption{Tuning Parameters And The Search Space For ResNet And MobileNet.}
    \begin{tabular}{lll}
    \toprule
        \multicolumn{2}{c}{\textbf{Shared Variables}} & {\textbf{Search Space}} \\
    \midrule
        \multicolumn{2}{l}{Learning Rate} & (0.001, 1)\\
        \hline
        \multicolumn{2}{l}{Epoch} & (50, 200)\\
        \hline
        \multicolumn{2}{l}{Batch} & (64, 360)\\
        \hline
        \multicolumn{2}{l}{Momentum} & (0, 0.999)\\
        \hline
        \multicolumn{2}{l}{Weight Decay} & ($1\times10^{-6}$, 0.999)\\
    \midrule
    \multicolumn{2}{c}{\textbf{Branch Variables}} & {\textbf{Nested Variables}} \\
    \midrule                                
    \multirow{2}{*}{Network Type} & {ResNet}  & {Depth} = \{{18}, {34}, {50} or {101}\} \\ 
         & {MobileNet} & {Multiplier} = \{{0.25}, {0.5} or {1.0}\} \\
    \bottomrule
    \end{tabular}

    \label{tab:Resnet18TunePara}
\end{table}

To illustrate the benefits of the proposed method, we test the proposed method in a series of image classification experiments spanning a wide range of hyperparameters: learning rate, epoch, batch size, momentum, weight decay, network type, and network setting. Among them, network type and network setting are categorical and nested variables. The detailed parameter space is shown in Table \ref{tab:Resnet18TunePara}. We benchmark our method with a direct search for the best hyperparameters for training neural networks on two popular datasets, CIFAR-10 and CIFAR-100 \citep{krizhevsky2009learning}. CIFAR-10 contains 60,000 $32\times 32$ natural RGB images in 10 classes. Each class has 5000 training images and 1000 testing images. CIFAR-100 is like CIFAR-10, except it has 100 classes. Each class has 500 training images and 100 testing images. For network types, ResNet \cite{he2016deep} and MobileNet \cite{sandler2018mobilenetv2} are two popular network families on various tasks: classification, detection, tracking, etc. Each CNN model is trained on an NVIDIA V100 GPU with PyTorch 1.4 \cite{paszke2017automatic} .

\captionsetup[figure]{font=small}
\begin{figure*}[!h]
\vspace{6pt}
     \centering
     \includegraphics[scale=0.12]{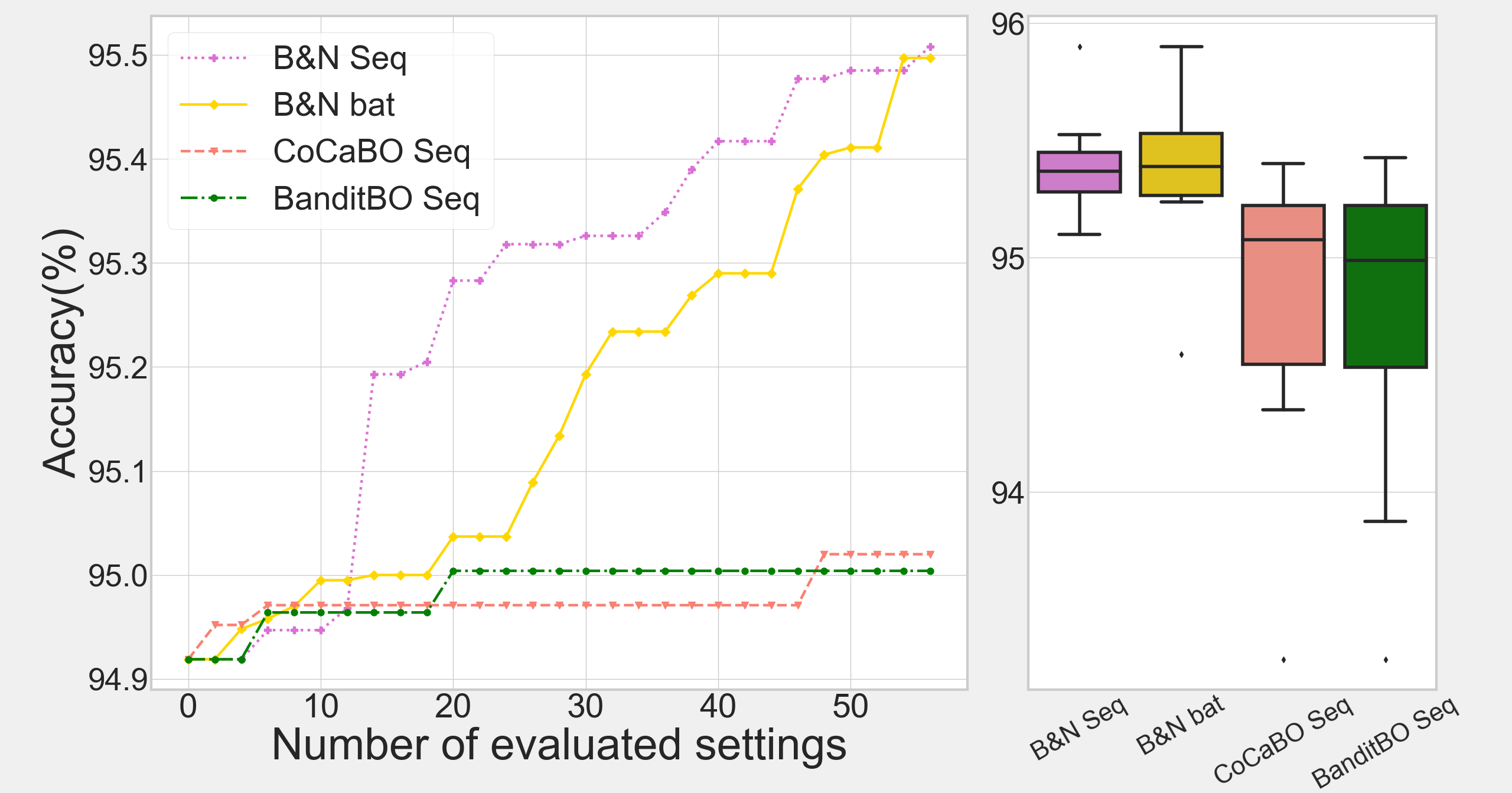}
     \centering
     \includegraphics[scale=0.12]{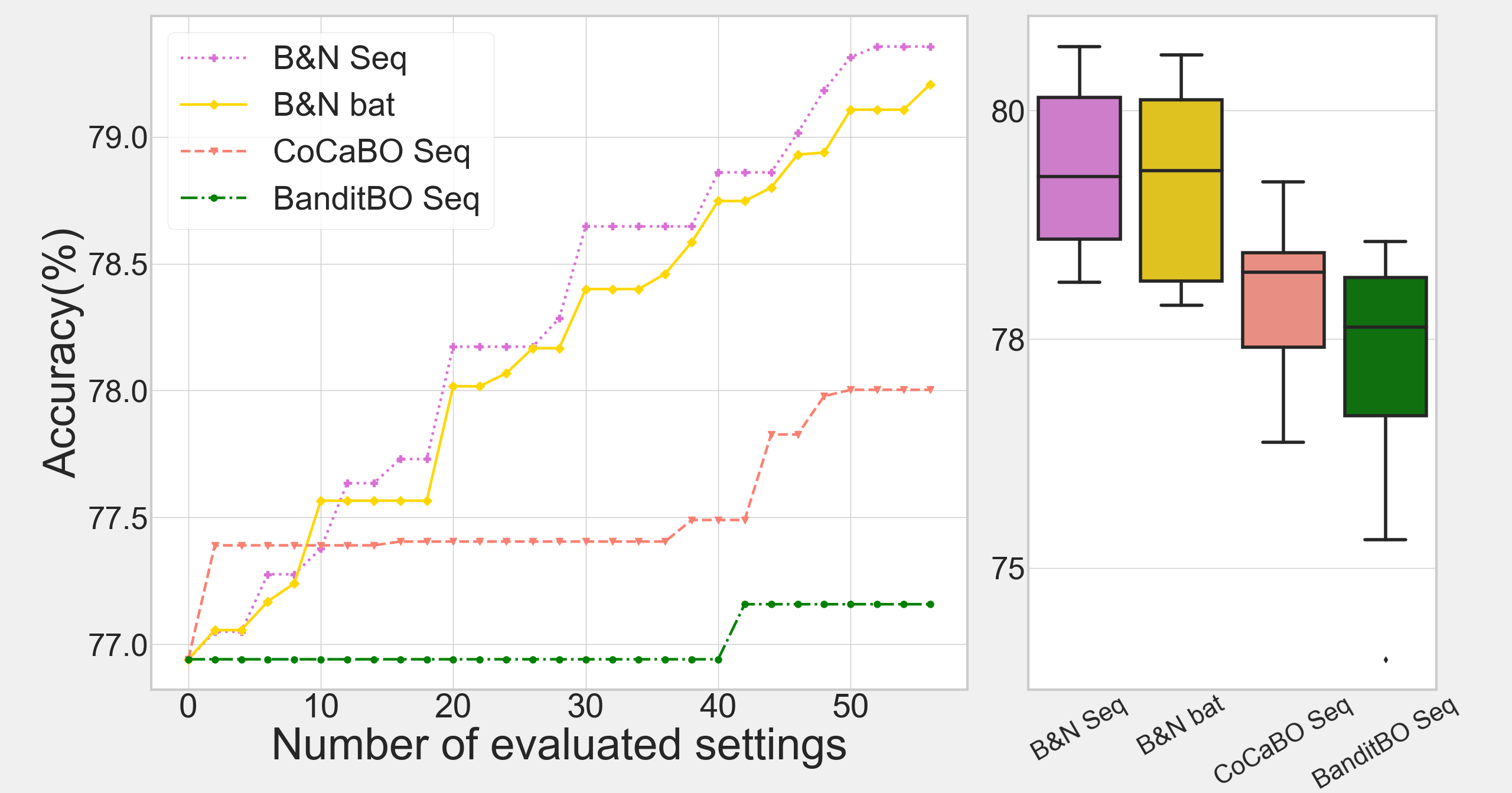}
    \vspace{6pt}
    \caption{Optimal tuning for CNN networks using datasets, CIFAR-10 (left) and CIFAR-100 (right). The proposed B$\&$N is implemented based on the sequential and the batch procedure with batch size 8. The performance is compared with CoCaBO and BanditBO using the sequential procedure, which is the most efficient alternative found by numerical studies. Starting from 28 initial settings randomly generated Latin hypercube designs, 56 additional settings are evaluated. The plots on the left show the best prediction accuracy as the search progresses. For each method, the box plots on the right are the optimal prediction accuracy summarized from 10 replicates at the end of the search. It appears that the B$\&$N procedures outperform CoCaBO and BanditBO, and the sequential procedure of B$\&$N converges to a stationary point faster than the batch procedure.}
    \label{fig:Realdata}
\end{figure*}

We compare B$\&$N with CoCaBO and BanditBO which appear to be the best alternatives found in Section 3.
For both methods, the analysis starts from the same initial design with 28 settings chosen by a randomly generated space-filling design called Latin hypercube design (LHDs) from the parameter space\citep{mckay2000comparison}. The total number of additional parameter combinations is set to be 56. Similar to Section 3, both sequential and batch procedures are considered for B$\&$N. The batch procedure includes eight additional points in each iteration, and seven iterations are performed. The experiment is repeated 10 times with different randomly generated LHDs, and the performance is summarized in Figures \ref{fig:Realdata}. 
The first and third plots in Figure \ref{fig:Realdata} show that the B$\&$N method can effectively propose new settings to obtain better accuracies as compared with CoCaBo and BanditBO under the same computational budgets. 
The second and fourth plots in Figure \ref{fig:Realdata} summarize the optimal prediction accuracy found at the end of the search. 
The second plot in Figures \ref{fig:Realdata} shows that B$\&$N method provides smaller variances, which implies that B$\&$N could achieve the desired stability while achieving better accuracy in CIFAR10 experiments. 
Comparing the sequential procedure with the batch procedure, the empirical results of both synthetic and real data indicate faster convergence rates by the sequential procedure. This results suggest that actively adding points by a reliable model can improve the search efficiency. 
Further studies are developed in the next sub-section to investigate the effects of the hyperparameters on the prediction accuracy.

Next, in Table \ref{tab:ExOptParams}, we highlight several optimal parameter settings found by B$\&$N based on a total of 84 evaluations of parameter combinations. For each dataset, two different parameter settings are demonstrated, and their corresponding classification accuracy is shown in the second row. The first setting for each dataset corresponds to the best performance found in 10 different initial designs. It appears that the proposed method provides a systematic search mechanism so that, even for a relatively complex problem such as CIFAR-100, a promising tuning parameter setting can be found within a limited computational effort. The second setting for each dataset corresponds to sub-optimal settings but with relatively low computational costs, i.e., smaller networks, smaller numbers of \textit{Epoch}, or \textit{Batch}. They are attractive in practice due to their computational advantages but not commonly recognized as promising settings. 
While we note that the accuracy from different papers are not directly comparable due to the use of different optimization and regularization approaches, it is still instructive to compare this result to others in the literature. Our best result on CIFAR-100 (80.70\%) is better than those of the recent paper with Knowledge Distillation method \citep{yuan2020revisiting} trained in the same network (79.43\% ResNet50) or a much larger network architecture (80.62\% DenseNet121 \cite{huang2017densely}). 
Inspired by these results, we will further develop more general objective functions that can incorporate other practical issues including computational cost.

\begin{table}[!t]
    \small
    \centering
    \caption{Examples Of The Optimal Settings Found By B$\&$N.}
    \begin{tabular}{lcccc}
    \toprule
     & \multicolumn{2}{c|}{\textbf{CIFAR-10}} &  \multicolumn{2}{c}{\textbf{CIFAR-100}} \\
    \midrule
     \textbf{Accuracy} & {\bf 95.92} & {\bf 95.66} & {\bf 80.70} & {\bf 79.61} \\
    \midrule
       Learning Rate  &  0.0113 &  0.0143	 & 0.0683 & 0.1866 \\
        \hline
         Epoch        & 185 & 141 & 185 & 180 \\
        \hline
         Batch        & 66 & 64 & 179 & 93\\
        \hline
         Momentum     & 0.75 & 0.56 & 0.57 & 0.49 \\
        \hline
         Weight Decay  & 0.0013 & 0.0193 & 0.0030 & 0.0017 \\
        \hline
         Network Type  & ResNet & ResNet & ResNet & ResNet\\
         \hline
         Depth  & 50 & 34 & 50 & 34\\ 
    \bottomrule
    \end{tabular}
 
    \label{tab:ExOptParams}
\end{table}

\captionsetup[figure]{font=small}
\begin{figure*}
\setlength\abovecaptionskip{0pt}
\begin{center}

\includegraphics[scale=0.18]{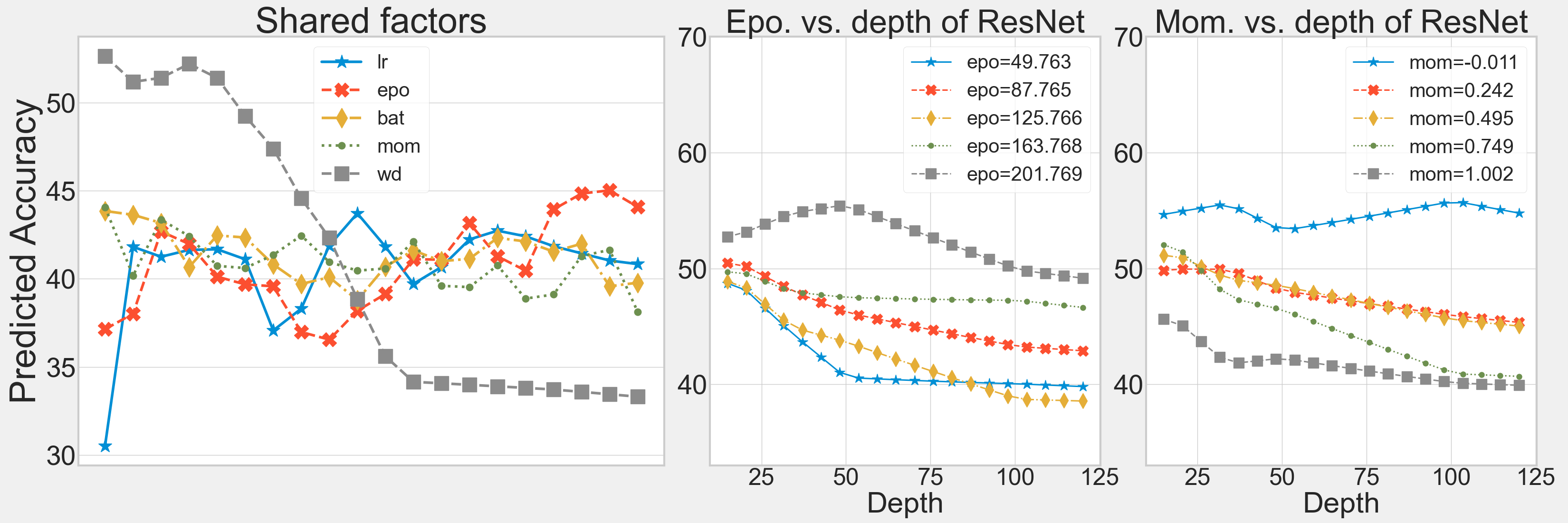}   
\end{center}
\vspace{8pt}
\centering
   \caption{Sensitivity Analysis. The left panel is the marginal main effects of the five shared variables. Weight decay, denoted by `wd', has the most significant decreasing effect on accuracy and an extremely low learning rate can lead to a low prediction accuracy. The middle panel is an illustration of the interaction effect between epoch and depth when ResNet is implemented. The setting of depth shows a slightly decreasing effect for smaller numbers of epoch but the effect becomes concave for larger numbers of epoch. The right panel shows the interaction between depth and momentum where the effect of depth increases slightly for a smaller momentum but decreases for a larger momentum.}
\label{fig:SensiAlys_3inrow}
\end{figure*}

\subsection{Sensitivity Analysis}
To explore the effects of the hyperparameters on the prediction accuracy, a sensitivity analysis using the Monte Carlo approach is performed based on CIFAR-100 with the proposed B$\&$N model \citep{sobol1993sensitivity, santner2003design}. The sensitivity analysis studies how sensitive the learning accuracy is to changes in the hyperparameters and evaluates which hyperparameters are responsible for the most variation in prediction accuracy. Three sensitivity plots are highlighted in Figures \ref{fig:SensiAlys_3inrow}. On the left panel of Figure \ref{fig:SensiAlys_3inrow}, the marginal main effects of the five shared variables are illustrated, where learning rate is denoted by `lr', epoch is denoted by `epo', batch size is denoted by `bat', momentum is denoted by `mom', and weight decay is denoted by `wd'. This plot shows that the setting of weight decay has a dominating effect on prediction accuracy as compared to the other shared variables, and a smaller weight decay is more likely to improve accuracy. This analysis also suggests that an extremely low learning rate should be avoided. Furthermore, an extensive two-factor interaction study is performed, and the detailed plots are given in the supplemental material. Among them, we highlight two significant interactions illustrated in the middle and the right panel of Figure \ref{fig:SensiAlys_3inrow}. In the middle panel, there is an interacting effect between the setting of epochs and the setting of depth when the ResNet structure is selected. For smaller numbers of epochs, such as the blue line for epo=49.763, the setting of depth shows a slightly decreasing effect. On the other hand, the effect becomes concave when larger numbers of epochs are implemented, such as the gray line for epo=201.769, and it reaches a higher prediction accuracy when the depth is chosen to be around 50. This is consistent with our findings in the optimal setting in Table \ref{tab:ExOptParams} where the number of epochs is relatively large (i.e., 185) and the depth is 50. On the right panel of Figure \ref{fig:SensiAlys_3inrow} is an interaction plot between the setting of momentum and the setting of depth where the effect of depth increases slightly for a smaller momentum (such as the blue curve) but decreases for a larger momentum (such as the gray curve). Generalizations to a broader range of datasets are yet to be explored; however, the findings from this sensitivity analysis shed light on the complexity of interactions in hyperparameter optimization.

%% file: 0_paper/content/s5_conclu.tex
\section{CONCLUSIONS AND OPEN PROBLEMS}

In this paper, we introduce the notion of branching and nested parameters which captures a conditional dependence commonly that occurs in parameter tuning for deep learning applications. A new kernel function is proposed for branching and nested parameters, and sufficient conditions are derived to guarantee the validity of the kernel function and the resulting reproducing kernel Hilbert space. Furthermore, a unified Gaussian process model and an expected improvement criterion are developed to achieve efficient global optimization when branching and nested tuning parameters are involved. The convergence rate of the proposed Bayesian optimization framework is proven under the continuum-armed-bandit setting, which provides an analogy to the convergence results with quantitative parameters.

An efficient initial design with branching and nested parameters is essential for Bayesian optimization. Therefore, an interesting area that deserves further study is the design of initial experiments with branching and nested parameters. Given the promising results observed by the implementation of randomly generated Latin hypercube designs in the current study, it will be interesting and important to further extend the idea of space-filling for the design of branching and nested parameters.

%% file: 1_supplementary/content/s8_supple_bkup.tex
\section{PROOFS}
\subsection{Proof of Theorem 1}
    \begin{theorem}
    \label{thm:validkernel}
    Suppose that there are $g^b_j$ levels in the nested variable $v_j^{b}$ which is nested within the branching variable $z_k=b$, for any $b\in\{1,2,\dots,l_k\}$. The kernel function in (4) is symmetric and positive definite if the hyperparameter $\boldsymbol{\phi}_k$ satisfy:
    \begin{align}
               \min_{b} \left[ \exp(-\phi^b_{kj})+\left(1-\exp(-\phi^b_{kj})\right)/g^b_j\right]
               \geq 
               \exp(-\gamma_k), \nonumber
     \end{align}          
    for all $j \in \{1,2,\ldots, m_k\}$ and $k\in\{1,2,\ldots, q\}$. 
    \end{theorem}
\begin{proof}
First of all, the matrix is clearly symmetric. 
To show it is positive definite, consider a pair of branching and nested variable, $(z_k, v_j^{z_k})$. There are in total $g(v_j^{z_k}) := \sum_{b=1}^{l_u}g_b(v_j^{z_k})$ possible outcomes. For simplicity, we use $l$, $g$ and $g_b$ to stand for $l_u$, $g(v_j^{z_k})$ and $g_b(v_j^{z_k})$ respectively. Then compute the correlation matrix, $\*T$, formed by these $g$ outcomes. By rearranging the order, $\*T$ can be written as a block matrix with $l\times l$ blocks:
    \begin{align}
    \label{corrmat:GCS}
        \*T =
        \begin{pmatrix}
            \*W_1 & \*B_{1,2} & \dots & \*B_{1,l}\\
            \*B_{2,1} & \*W_2 & \dots & \*B_{2,l}\\
            \vdots & \vdots & \vdots & \vdots \\
            \*B_{l, 1} & \*B_{l, 2} & \dots & \*W_{l}
        \end{pmatrix}.
    \end{align}
    In (\ref{corrmat:GCS}), $\*W_b$ is the within-group correlation which captures the correlation matrix of observations with $z_k=b$. The off-diagonal blocks, $\*B_{i,j}$, are the between-group correlation which contains the correlation of pairs of observations with $z_k$ value equals to $i$ and $j$ respectively. 
    Denote $\overline{W_i}$ and $\overline{B_{i,j}}$ as the the average of all elements of $\*W_i$ and $\*B_{i,j}$ respectively and let
    \begin{align}
        \widetilde{\*T}= 
        \begin{pmatrix}
            \overline{W_1} & \overline{B_{1,2}} & \dots & \overline{B_{1,l}}\\
            \overline{B_{2,1}} & \overline{W_2} & \dots & \overline{B_{2,l}}\\
            \vdots & \vdots & \vdots & \vdots \\
            \overline{B_{l,1}} & \overline{B_{l,2}} & \dots & \overline{W_{l}}\\
        \end{pmatrix}.
    \end{align}
    $\*T$ is called a Generalized Compound Symmetric(GCS) matrix in \cite{roustant2020group} and it has been proved that  $\*T$ is  positive definite if and only if both $\widetilde{\*T}$ and   $\*W_i$ are positive definite.
    Note that, based on the correlation function assumed in equation (5) of the main manuscript, we have $\*B_{i,j}\equiv\*B$ and $\overline{B_{i,j}}\equiv\overline{B}$. Hence we have 
    \begin{align}
        \widetilde{\*T} = \overline{B}\bm{1}_{l}\bm{1}_{l}^T + \textit{diag}(\overline{W_1}-\overline{B}, \dots, \overline{W_{l}}-\overline{B}).
    \end{align}
    It is now clear that $\widetilde{\*T}$ is positive definite if each $\overline{W_b}-\overline{B}$ is positive definite, which is satisfied by our condition (6).
\end{proof}
\vfill

\subsection{Proof of Theorem 2}
\begin{theorem}{Assume that $f$ follows a Gaussian process with the kernel function defined in (4), where $R_{\boldsymbol{\theta}}$ is a Mat\'ern kernel with the smoothness parameters $\nu>0$ and some $\alpha\geq 0$ depends on $\nu$. Let $y^*_n=max_{1\leq i\leq n}y_i$, we have
\begin{align*}
    \sup_{\|f\|_{\mathcal{H}(\mathcal{S})}\leq S}\mathbb{E}&\left\{y^*_{n}-\max_{\mathbf{x}\in\mathcal{X}} f(\mathbf{x})|D_n\right\}\\
    =& \mathcal{O}\left(L^{\nu/d}(n/\log{n})^{-\nu/{d}}(\log n)^{\alpha}\right).
\end{align*}
}
\end{theorem}

\begin{proof}
Denote $L$ as the number of all possible choices for branching and nested variables, and $\widetilde{n}:=\frac{n\epsilon}{4L}$ and $n_l$ as the number of points whose $(z,v)=l$. Suppose $k>L$ initials points are selected independently from $f$, and $n>2k$. Then consider the events $A_{n,l}^c:=\{\tilde{n}$ points are selected uniformly random with $(z,v)=l$ among $x_{k+1}\dots x_n\}$ and $A_n=\bigcap_{l=1}^L A_{n,l}$. By Chernoff inequality and a union bound, we have
    \[\mathbb{P}\{A_n^c\}\leq \sum_{l=1}^L\mathbb{P}\{A_{n,l}^c\}\leq L\exp(-\frac{\epsilon n}{16L}).\]

Denote $B_n:=\{$at least one point among $x_{n+1}\dots x_{2n}$ is selected by EI$\}$, then $\mathbb{P}\{B_n^c\}=\epsilon^n$. In total, we have $\mathbb{P}\{A^c_n\bigcup B^c_n\}$ decays in $n$ exponentially.

For an input setting $x:=(w, z, v)$, denote a compact subset $\mathcal{X}_w\subset \mathbb{R}^d$ as the design space of shared factor $w$, and denote the mesh norm $h_{n}=sup_{w\in\mathcal{X}_w}\min_{1\leq i \leq {n}}\|w-w_i\|_2$. Set $r_n:=(n/\log n)^{-\nu/d}(\log n)^{\alpha}$, then since $h_n$ is non-incresing in $n$ and by Lemma 12 in \cite{bull2011convergence}, one has the following inequalities:
\[
\mathbb{P}\{E^c_n\bigcap B_n\bigcap A_n\}\leq
\sum_1^L\mathbb{P}\{E^c_{n,l}\bigcap B_n\bigcap A_n\}\leq
C_1 Lr_{\tilde{n}}
\]
where $E_n:=\{h_n\leq C_0(n/\log n)^{-1/d}\}$ and $C_0, C_1$ are constants.%

For $E_n\bigcap B_n \bigcap A_n$, there is a point $m$, where $n<m\leq 2n$, chosen by EI and $w.l.o.g$ assume $(z_m, v_m)=l$. By Lemma 8 in \cite{bull2011convergence}, we have
\begin{align*}
    y^*-y^*_{m-1} &\leq EI_{m-1}(x^*) + R_m s_{m-1}(x_m|\widehat{\theta_m})\\
    &\leq y^*_{m}-y^*_{m-1} \\
    & \;\;\; + (R_m+\widehat{\sigma^2_m})s_{m-1}(x_m|\widehat{\theta_m})+R_m s_{m-1}(x_m|\widehat{\theta_m}),
\end{align*}
where $S_n=\|f\|_{\mathcal{H}_{\hat{\theta}_n}}$ and $\widehat{\sigma^2_n}=(Y_n-\widehat{\mu_n})^T\widehat{K_n^{-1}}(Y_n-\widehat{\mu_n})$. By Lemma 4 and Corollary 1 from \cite{bull2011convergence}, both $S_n$ and $\widehat{\sigma^2_n}$ are bounded by $S'=S\prod_{i=1}^d(\theta_i^U/\theta_i^L)$. Moreover, denote $D_{n_l}$ as the slice of observations with $D_n|_{(z,v)=l}$, it follows that $s_{m-1}(x_m|\widehat{\theta_m}) \leq \max_L s_{n_l}(x_m|\widehat{\theta_m})$. By \cite{narcowich2003refined}, we have 
$s_{n_l}(x|\*\theta)=\mathcal{O}(M(\*\theta)h_{n_l}^\nu)$ uniformly for $\*\theta$, where $\theta_i^L\leq \theta_i\leq\theta_i^U$, $1\leq i\leq d$ and $M(\*\theta)$ is continuous in $\*\theta$. Also note that $(n/\log n)^{-1/d}$ decreases in $n$, hence $s_{m-1}(x_m|\widehat{\theta_m}) \leq \max_L s_{n_l}(x_m|\widehat{\theta_m})\leq C_2 r_{\tilde{n}}$, where $C_2$ depends on $D_w$, $R$, $C_0$, $\*\theta^L$ and $\*\theta^U$.
Therefore, $y^*-y^*_{m-1} \leq y^*_{m}-y^*_{m-1} + 3C_2 S' r_{\widetilde{n}}$ and $y^*-y^*_{2n}\leq y^*-y^*_{m}\leq 3C_2 S'r_{\widetilde{n}}$.

Lastly, for the event of $E_n^c$, $y^*-y^*_{2n}\leq 2\norm{f}_{\infty}\leq 2S$. Combine the arguments above, we have
\begin{align*}
    \mathbb{E}\{y^*_{2n+1}-y^*\} 
    &\leq \mathbb{E}\{y^*_{2n}-y^*\} \\
    &\leq 2S \left(\mathbb{P}\{E_n^c\bigcap A_n \bigcap B_n\}\right. \\
    & ~~~~ \left.+ \mathbb{P}\{A_n^c \bigcup B_n^c\}\right) + 3C_2S'r_{\widetilde{n}}\\
    &= 2S\left(C_1L r_{\tilde{n}} + L\exp(-\frac{\epsilon n}{16 L}) + \epsilon^n\right) \\
    & ~~~~ + 3C_2S^{'}r_{\tilde{n}}\\
    &= \mathcal{O}(L^{v/d}(n/\log n)^{-\nu/d}(\log n)^\alpha).
\end{align*}
The results follows as the bound is uniform in $f$ with $\|f\|_{\mathcal{H}_{\theta^U}}\leq S$.
\end{proof}

\section{SENSITIVITY ANALYSIS}
In Section 4.2 of the main manuscript, a sensitivity analysis is performed. In this subsection, two-factor interaction plots for the shared factors are provided in Figure \ref{fig:sensialys-sharedfac}. In general, there are two slightly larger two-factor interactions: one is between momentum and batch size (the third picture in the fourth row), and the other is between weight decay and epoch (the second picture in the last row). It appears that the effect of momentum is mild for smaller batch sizes but becomes positive for larger batch sizes. On the other hand, the effect of weight decay, especially for smaller values of the weight decay, changes with respect to the epoch settings.  

\begin{figure*}
\centering
    \subcaptionbox{
    Two-factor interaction effect between `learning rate' and the other shared variables
    }{
    \includegraphics[scale=0.13]{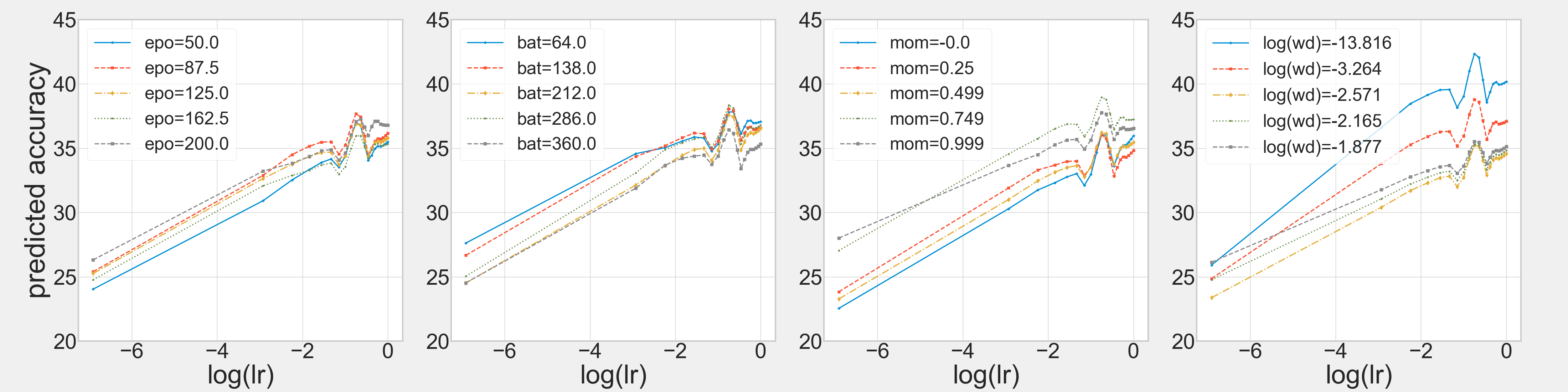}
    }
    \subcaptionbox{
    Two-factor interaction effect between `epoch' and the other shared variables
    }{
    \includegraphics[scale=0.13]{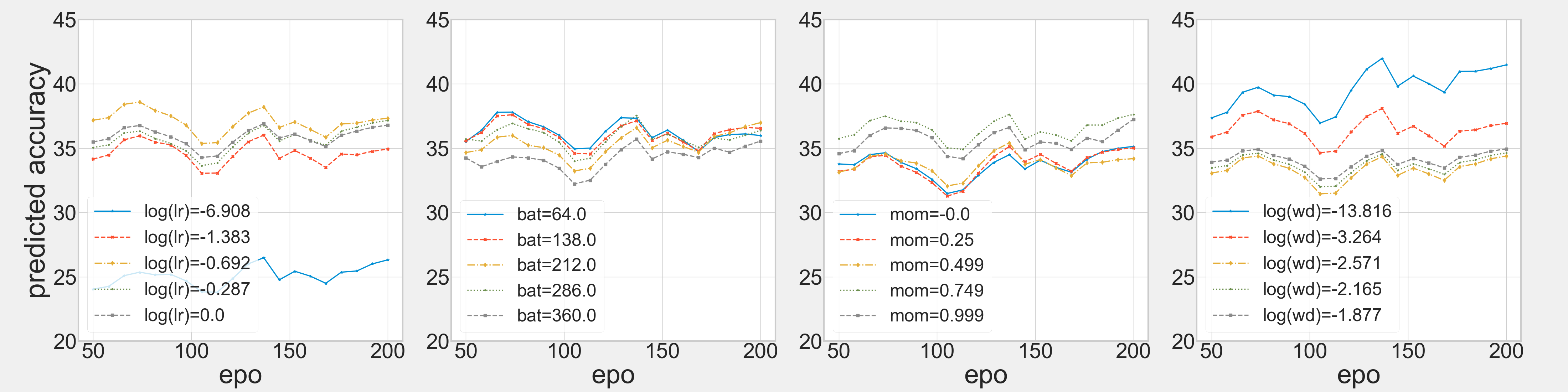}
    }
    \subcaptionbox{
    Two-factor interaction effect between `batch' and the other shared variables
    }{
    \includegraphics[scale=0.13]{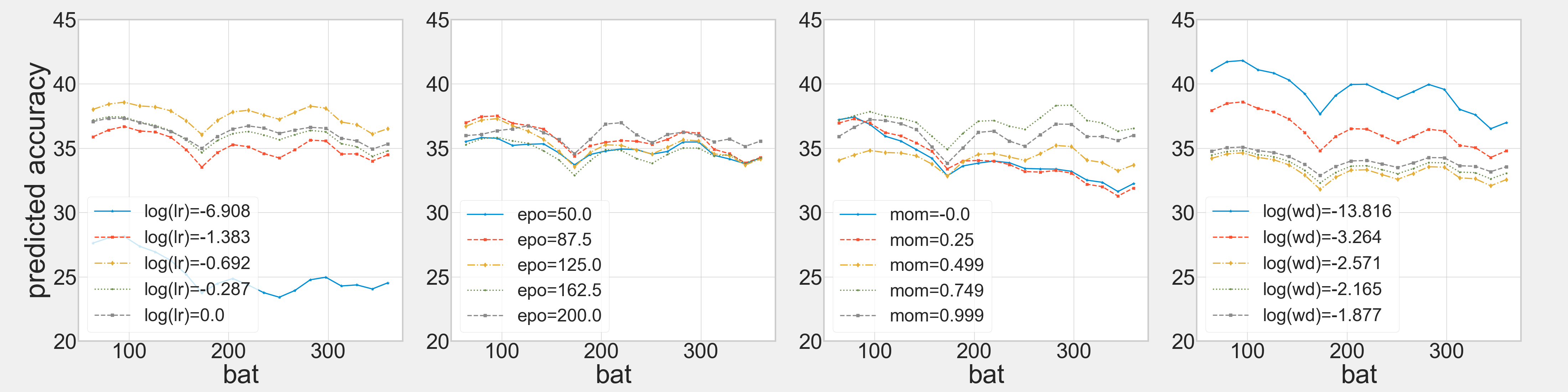}
    }
    \subcaptionbox{
    Two-factor interaction effect between `momentum' and the other shared variables
    }{
    \includegraphics[scale=0.13]{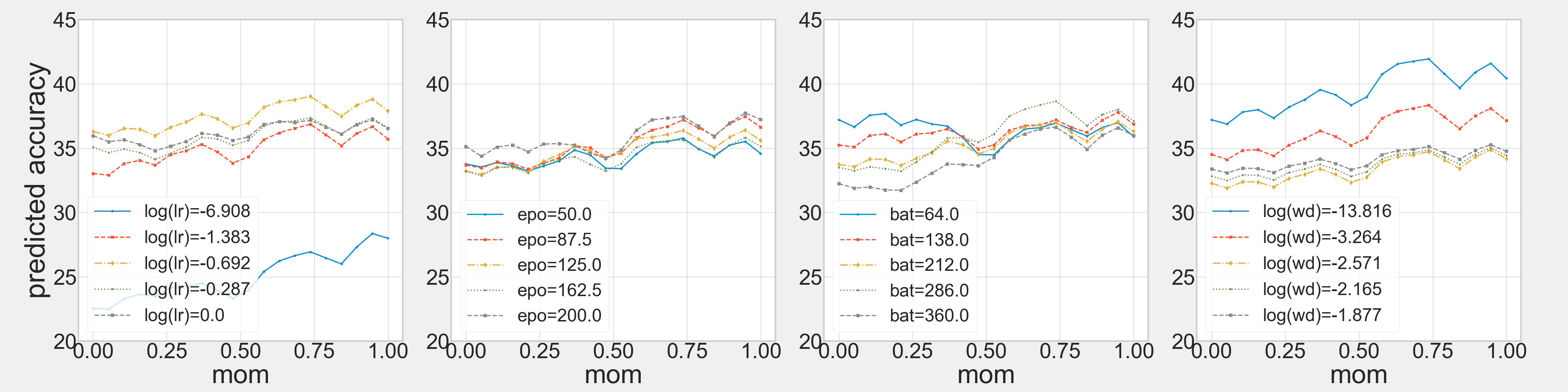}
    }
    \subcaptionbox{
    Two-factor interaction effect between `weight decay' and the other shared variables
    }{
    \includegraphics[scale=0.13]{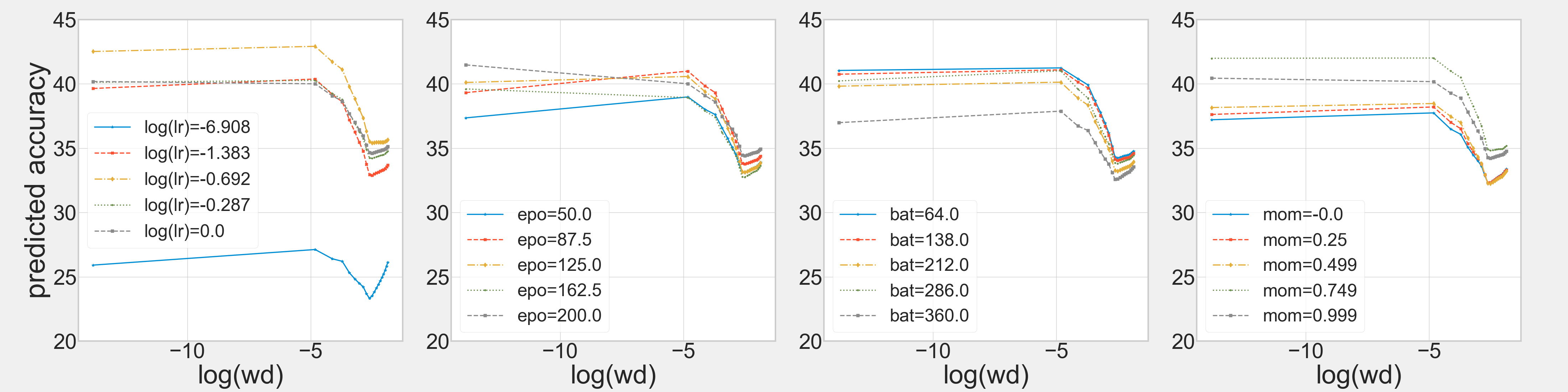}
    }
\caption{Two-factor interaction plots for the five shared variables, where learning rate is denoted by `lr', epoch is denoted by `epo', batch is denoted by `bat', momentum is denoted by `mom', and weight decay is denoted by `wd'.}
\label{fig:sensialys-sharedfac}
\end{figure*}

\subsection{Analysis of optimizer of Deep Neural Network}

Additional to the analysis of the two network types in the main text, the proposed method is also applied to the analysis of the optimizer in deep neural networks. In general, the training of deep neural network involves two types of hyperparameters: one type relates to building the structure of models, such as the number of hidden layers and the type of activation functions; the other type is directly associated with accuracy, such as the learning rate of an optimizer, batch size, weight decay, momentum, and schedulers. In practice, there has been more focuses on tuning the hyperparameters whose settings are believed to directly determine the accuracy. However, less emphasis has been given to understanding how the model structure and hyperparameters affect the accuracy and furthermore how the two types of hyperparameters interactively affect accuracy. This is partially due to the lack of an efficient model that can incorporate the complex conditional dependence in the settings of the model structure. Therefore, the proposed Bayesian optimization B$\&$N is applied to address this problem and shed light on the tuning of popular neural network architectures like ResNet and its variants.

    \begin{table}[!htb]
    \small
    \footnotesize      
    \begin{center}
    \caption{Tuning parameters and the search space for ResNet-18.} 
    \begin{tabular}{lll} %
    \toprule
        \multicolumn{2}{c}{\textbf{Shared Variables}} & {\textbf{Search Space}} \\
    \midrule
        \multicolumn{2}{l}{Learning Rate} & (0.001, 1)\\
        \hline
        \multicolumn{2}{l}{Epoch} & (100, 350)\\
        \hline
        \multicolumn{2}{l}{Batch} & (8, 350)\\
        \hline
        \multicolumn{2}{l}{Momentum} & (0, 0.999)\\
        \hline
        \multicolumn{2}{l}{Weight Decay} & ($1\times10^{-6}$, 0.999)\\
  \midrule
    \multicolumn{2}{c}{\textbf{Branch Variables}} & {\textbf{Nested Variables}} \\
    \midrule                                
    \multirow{2}{*}{Optimizer} & {SGD}  & {Scheduler} = \{{Cyclic}, {Cosine} or {Step}\} \\ 
         & {Adam} & {Scheduler} = \{{Step} or {Cosine}\} \\
    \bottomrule
    \end{tabular}
    \label{tab:Resnet18TunePara}
    \end{center}
    \end{table}

\begin{figure*}[h]
\setlength\abovecaptionskip{0pt}
\begin{center}
\small

\subcaptionbox{ResNet-18 on CIFAR-10 dataset.}{
\includegraphics[scale=0.11]{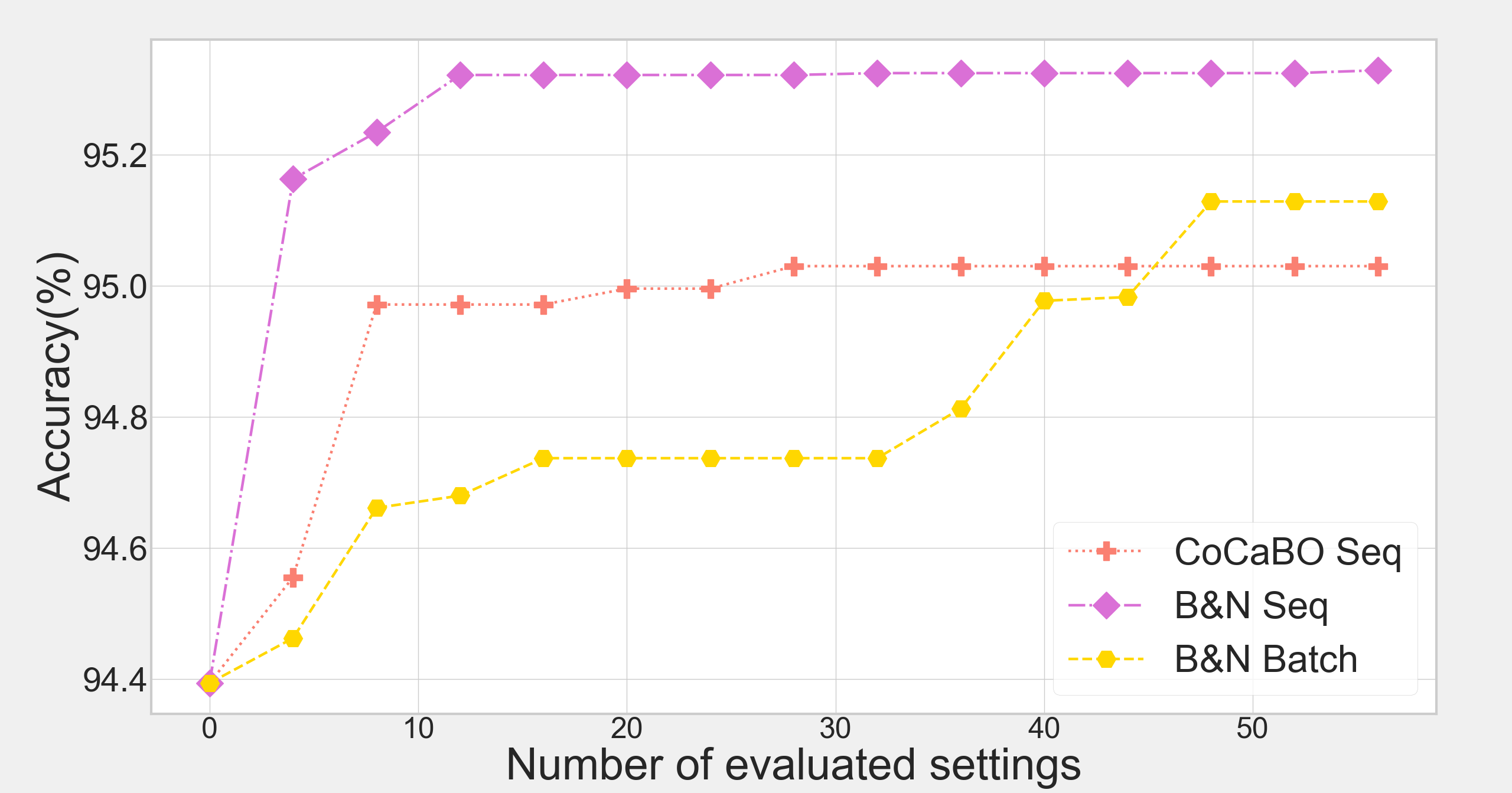}
\includegraphics[scale=0.11]{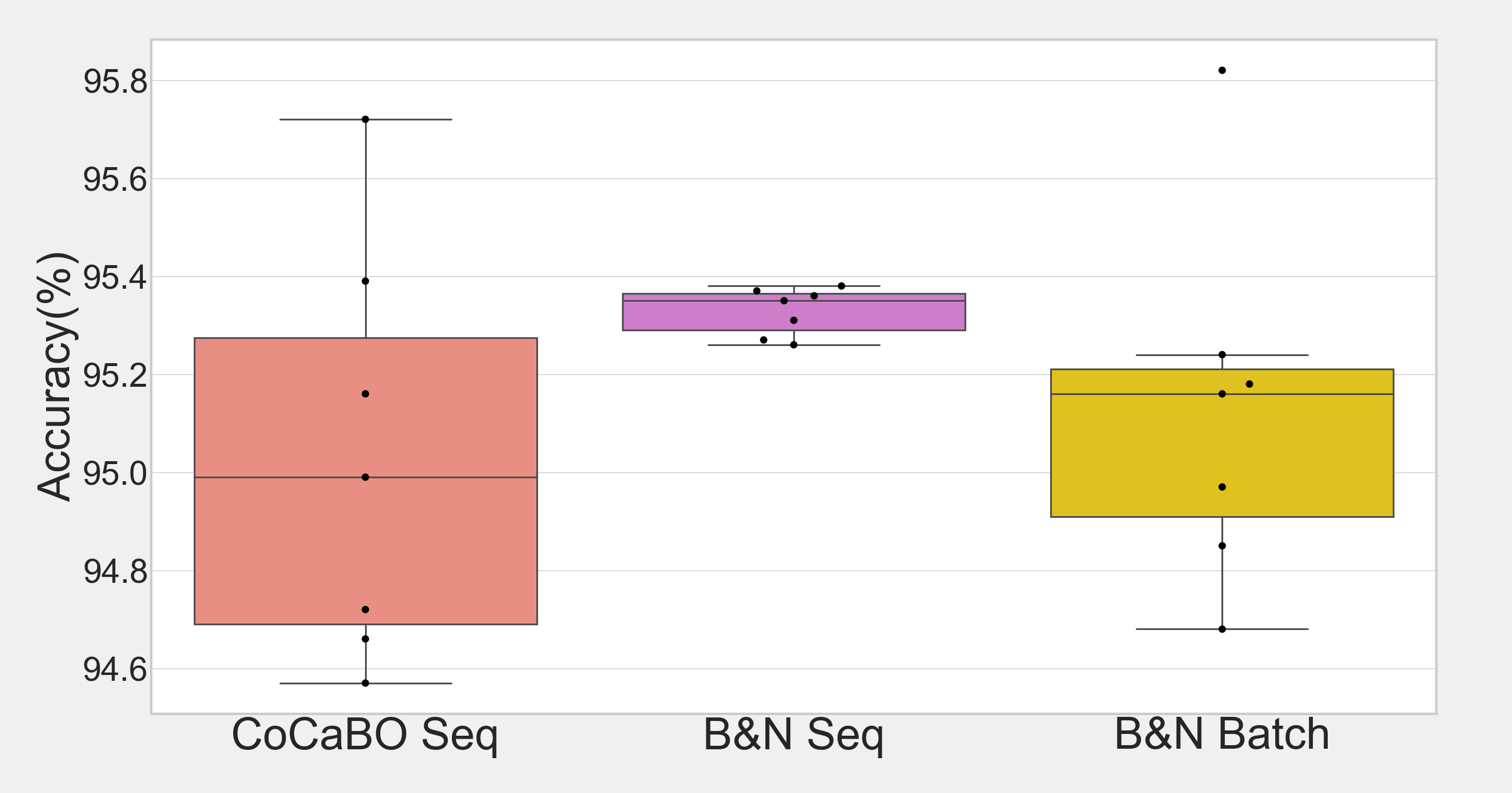}
}
\subcaptionbox{ResNet-18 on CIFAR-100 dataset.}{
\includegraphics[scale=0.11]{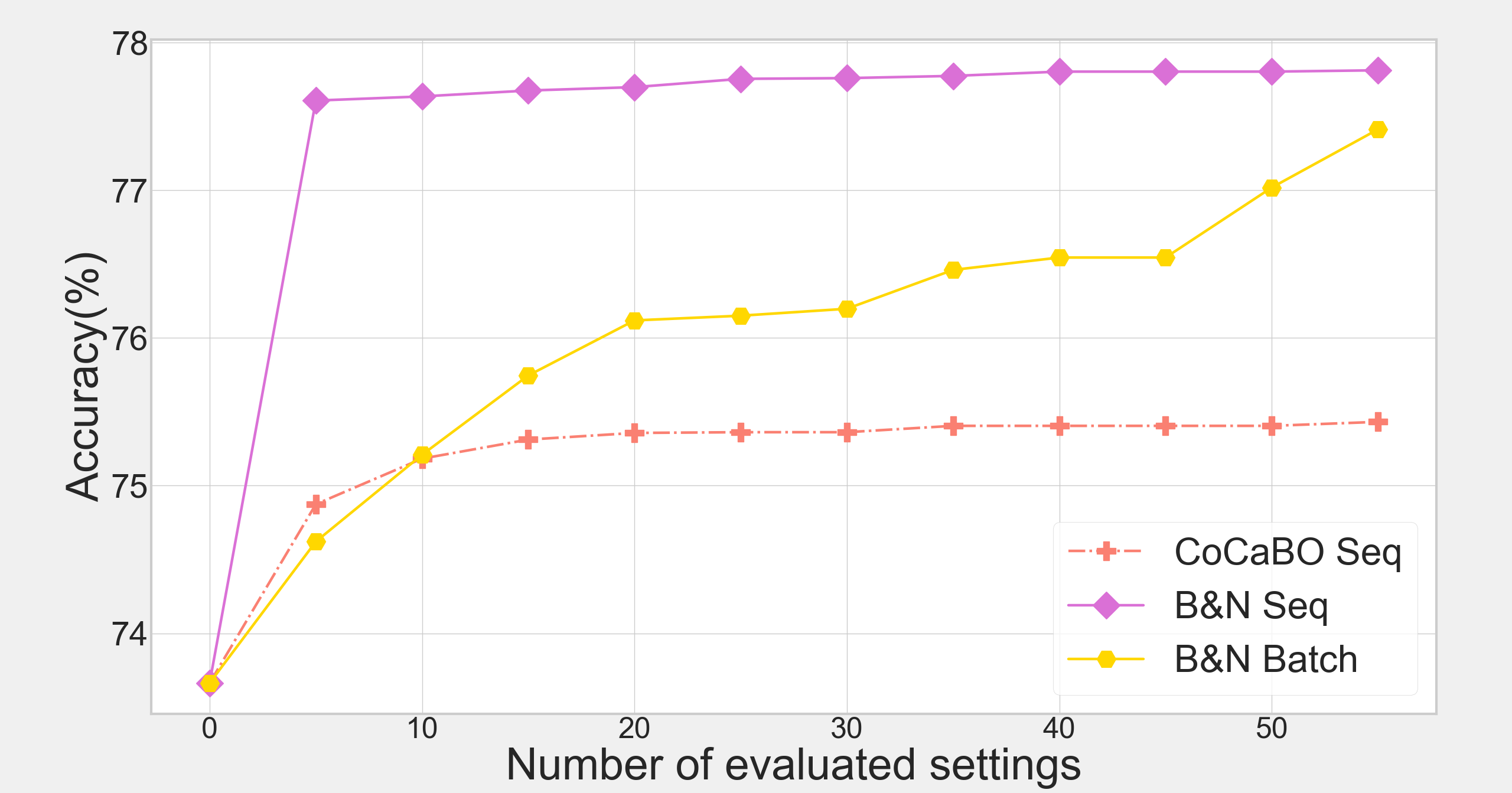}
\includegraphics[scale=0.11]{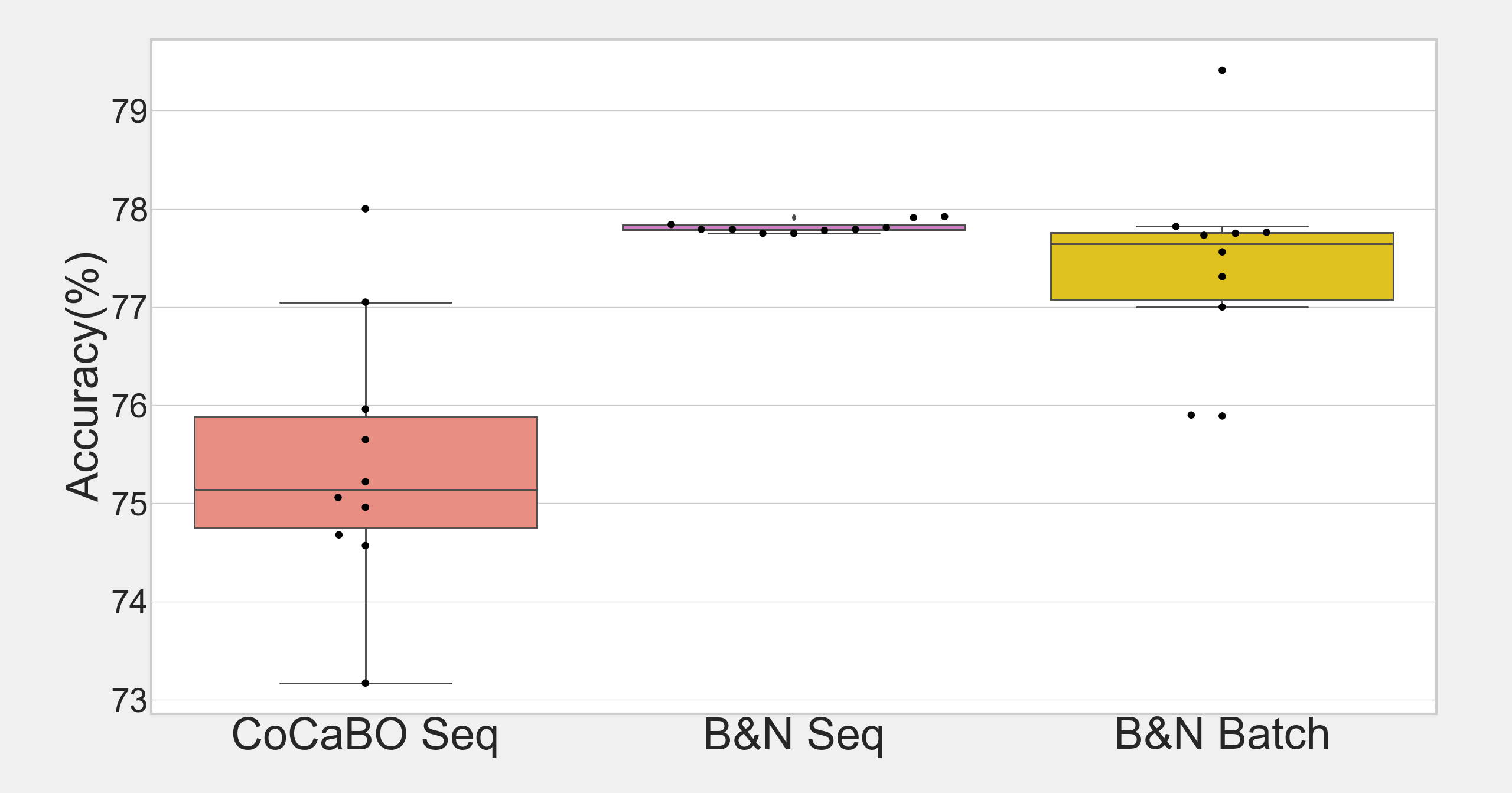}
}
\end{center}
      \vspace{6pt}
    \centering

   \caption{Optimal tuning for ResNet-18 using datasets CIFAR-10 (upper panels) and CIFAR-100 (lower panels). The proposed B$\&$N is implemented based on the sequential and the batch procedure with batch size 8. The performance is compared with CoCaBO using the sequential procedure, which is the most efficient alternative found in numerical studies. Starting with 10 initial settings randomly generated Latin hypercube designs, 56 additional settings are evaluated. The plots on the left show the best prediction accuracy as the search progresses. For each method, the box-plots on the right are the optimal prediction accuracy summarized from 20 replicates at the end of the search. It appears that the B$\&$N procedures outperform CoCaBO, and the sequential procedure of B$\&$N converges to a stationary point faster than the batch procedure.}

\label{fig:ResnetResult}\vspace{-0.6\baselineskip}
\end{figure*}

To illustrate the benefits of the proposed method, we test it in a series of image classification experiments spanning a wide range of hyperparameters: learning rate, optimizer, batch size, weight decay, scheduler, weight decay, and momentum. Among them, optimizer and scheduler are categorical and nested variables. The detailed parameter space is shown in Table \ref{tab:Resnet18TunePara}. We benchmark our method with a direct search for the best hyperparameters on training a neural network. A ResNet-18 is considered for two popular datasets, CIFAR-10 and CIFAR-100 \citep{krizhevsky2009learning}. ResNet is a most popular network family on various tasks: classification, detection, tracking, etc. CIFAR-10 contains 60,000 $32\times 32$ natural RGB images in 10 classes. Each class has 5000 training images and 1000 testing images. CIFAR-100 is like CIFAR-10, except it has 100 classes. Each class has 500 training images and 100 testing images.

We compare B$\&$N with CoCaBO \cite{ru2020bayesian} which appears to be the best alternative found in Section 3.
For both methods, the analysis starts from the same initial design with 10 settings chosen by a randomly generated space-filling design called Latin hypercube design (LHDs) from the parameter space\citep{mckay2000comparison}. The total number of additional parameter combinations is set to be 56. Similar to Section 3, both sequential and batch procedures are considered for B$\&$N. The batch procedure includes eight additional points in each iteration, and seven iterations are performed. The experiment is repeated 10 times with different randomly generated LHDs, and the performance is summarized in Figures \ref{fig:ResnetResult}. The left panel in Figure \ref{fig:ResnetResult} shows a faster convergence of the two B$\&$N methods as compared with CoCaBO, and the improvement over CoCaBO is much more significant for CIFAR-100, which is known to be challenging in parameter tuning. The box-plots on the right panel of Figures \ref{fig:ResnetResult} are summarized from the optimal prediction accuracy found at the end of the search. The B$\&$N methods have higher accuracy as compared with CoCaBO and consistently provide smaller variances, which implies that B$\&$N could achieve the desired stability while achieving better accuracy. Comparing the sequential procedure with the batch procedure, the empirical results both in synthetic function and real data indicate a considerably faster convergence to a stationary point by the sequential procedure. This result appears to confirm that actively adding points by a reliable model can significantly improve the search efficiency. Further studies will be rigorously developed in the future to investigate an optimal choice of batch size in the adaptive procedure.

\begin{table}[!htb] 
    \small
    \footnotesize        
    \begin{center}
  
    \begin{tabular}{l|cc|cc} %
    \toprule
     & \multicolumn{2}{c|}{\textbf{CIFAR-10}} &  \multicolumn{2}{c}{\textbf{CIFAR-100}} \\
    \midrule
     \textbf{Accuracy} & {\bf 95.38} & {\bf 95.26} & {\bf 79.41} & {\bf 77.75} \\
    \midrule
       Learning Rate  &  0.0016 &  0.0019	 & 0.0112 & 0.0026 \\
        \hline
         Epoch  & 350 & 101 & 350 & 138 \\
        \hline
         Batch  & 313 & 25 & 8 & 66\\
        \hline
         Momentum  & 0.37 & 0.98 & 0.55 & 0.71 \\
        \hline
         Weight Decay  & 0.2791 & 0.4146 & 0.0008 & 0.9811\\
        \hline
         Optimizer  & SGD & SGD & SGD & SGD\\
         \hline
         Schedule  & Cosine & Cyclic & Cosine & Cosine\\ 
    \bottomrule
    \end{tabular}
    \caption{Examples of the optimal settings found by B$\&$N.} 
    \label{tab:ExOptParams}
    \end{center}
    \end{table}

For a final evaluation, in Table \ref{tab:ExOptParams}, we highlight several optimal parameter settings found by B$\&$N based on a total of 66 evaluations of parameter combinations. For each dataset, two different parameter settings are demonstrated, and their corresponding classification accuracy is shown in the second row. The first optimal setting for each dataset corresponds to the best performance found in 10 different initial designs. It appears that the proposed method provides a systematic search mechanism so that, even for a relatively complex problem such as CIFAR-100, a promising tuning parameter setting can be found within limited computational effort. The second optimal setting for each dataset corresponds to sub-optimal settings but with relatively low computational costs, i.e., smaller numbers of \textit{Epoch} and \textit{Batch}. They are attractive in practice due to their computational advantages but are not commonly recognized as promising settings. 
While we note that the accuracy from different papers is not directly comparable due to the use of different optimization and regularization approaches, it is still instructive to compare this result to others in the literature. Our results on CIFAR-100 are better than that of the recent paper with Knowledge Distillation method (77.36\% \citep{yuan2020revisiting}) trained in the same ResNet-18 network.
Inspired by these results, we will further develop more general objective functions that can incorporate other practical issues including computational cost.